\documentclass[preprint,12pt]{elsarticle}

\usepackage{amssymb}
 \usepackage{amsthm}
\usepackage{amsmath} 
\usepackage{algorithm} 
\usepackage{algorithmic} 
\usepackage{booktabs}
\usepackage{subfigure}
\usepackage{multirow}
\usepackage[numbers]{natbib}
\newtheorem{theorem}{Theorem}

\journal{Neurocomputing}

\begin{document}

\begin{frontmatter}



\title{Robust Capped $\ell_{p}$-Norm Support Vector Ordinal Regression}


\author[mymainaddress,mysecondaryaddress]{Haorui Xiang}
\ead{haoruixiang@mail.nwpu.edu.cn}
\author[mymainaddress,mysecondaryaddress]{Zhichang Wu}
\ead{wks0774@mail.nwpu.edu.cn}
\author[mymainaddress]{Guoxu Li}
\ead{lgx2683109120@gmail.com}
\author[mymainaddress]{Rong Wang}
\ead{wangrong07@tsinghua.org.cn}
\author[mymainaddress,mysecondaryaddress]{Feiping Nie\corref{mycorrespondingauthor}}
\cortext[mycorrespondingauthor]{Corresponding author}
\ead{feipingnie@gmail.com}
\author[mysecondaryaddress]{Xuelong Li}
\ead{li@mail.nwpu.edu.cn}

\address[mymainaddress]{School of Computer Science and School of Artificial Intelligence, Optics and Electronics (iOPEN), Northwestern Polytechnical University, Xi'an 710072, Shaanxi, P. R. China.}
\address[mysecondaryaddress]{ Institute of Artificial Intelligence (TeleAI), China Telecom Corp Ltd, 31 Jinrong Street, Beijing 100033, P. R. China}


\begin{abstract}

Ordinal regression is a specialized supervised problem where the labels show an inherent order.
The order  distinguishes it from normal multi-class problem.
Support Vector Ordinal Regression, as an outstanding ordinal regression model, is widely used in many ordinal regression tasks.
However, like most supervised learning algorithms, the design of SVOR is based on the assumption that the training data are real and reliable, which is difficult to satisfy in real-world data.
In many  practical applications, outliers are frequently present in the training set, potentially leading to misguide the learning process, such that the  performance is non-optimal. 
In this paper,
we propose a novel capped $\ell_{p}$-norm loss function that is theoretically robust to both light and heavy outliers.
The capped $\ell_{p}$-norm loss can help the model detect and eliminate outliers during training process.
Adhering to this concept, we introduce a new model, Capped $\ell_{p}$-Norm Support Vector Ordinal Regression(CSVOR), that is robust to outliers.
CSVOR uses a weight matrix to detect and eliminate outliers during the training process to improve the robustness to outliers.
Moreover, 
a Re-Weighted algorithm algorithm 
which is illustrated  convergence by our theoretical results is proposed to effectively minimize the corresponding problem.
Extensive experimental results demonstrate that our model outperforms state-of-the-art(SOTA) methods, particularly in the presence of outliers.


\end{abstract}




\end{frontmatter}


\section{Introduction}
\label{}
Ordinal regression, also known as ordinal classification, is a machine learning task that involves predicting labels with a natural order, distinguishing it from standard multi-class classification tasks \cite{gutierrez2015ordinal,cruz2014metrics,pfannschmidt2020feature}. In ordinal regression, the order of the labels in the output space corresponds to the order in the input space, allowing models to utilize the inherent order information for more accurate predictions.
For instance, when predicting ages, a 9-year-old child's label is more likely to be closer to 8-year-old than to 7-year-old. This characteristic is challenging for standard multi-class methods, as they treat labels as unrelated categories
\cite{diaz2019soft,yan2014cost}.
Ordinal regression finds applications across various domains, including age prediction 
\cite{vargas2020cumulative,niu2016ordinal,5995437,tian2014comparative}, text classification \cite{li2022ordinalclip,qian2022contrastive}, medical research \cite{burkner2019ordinal,tang2023disease}, face recognition \cite{zhu2021convolutional, shin2022moving}, and social study \cite{pratiwi2019implementing,sun2023robust}.

In recent years, ordinal regression, a form of supervised learning where the output variable represents a ranking or rating scale, has gained increasing attention, leading to the proposal of many ordinal regression models \cite{shin2022moving}.
One successful model in this domain is Support Vector Ordinal Regression with Explicit Constraints (SVOREX), which is widely utilized in ordinal regression tasks \cite{chu2005new,zhong2023ordinal}.
SVOREX employs the ordinal hinge loss function and incorporates the ordinal relationship of classes into the optimization problem through explicit inequality constraints on the threshold vector $\boldsymbol{b}$ .
 It learns $K-1$ parallel planes to separate the $k$ classes.

As the volume and diversity of available data continue to expand, the likelihood of encountering outliers, such as sensor errors and mislabeled data, also increases \cite{wang2023enhanced}. These outliers can significantly impact data analysis and decision-making processes, highlighting the importance of robust techniques   \cite{zhang2020towards}.
However,  
like most supervised learning algorithms,
 SVOREX is designed under the assumption that the training data are real and reliable, which is often challenging to satisfy in real-world datasets.
Real-world datasets often contain numerous outliers, which are data points significantly different from others in the same class or data with incorrect labels.
The presence of these outliers can lead SVOREX to find sub-optimal solutions, resulting in a significant reduction in model performance.

To tackle this challenge, we propose a Capped $\ell_{p}$-Norm Support Vector Ordinal Regression (CSVOR) model that utilizes a capped $\ell_{p}$-norm ordinal hinge loss. The capped $\ell_{p}$-norm ordinal hinge loss, unlike the ordinal hinge loss used in SVOREX, is theoretically robust against both light and heavy outliers.
Since outliers often have large residuals  \cite{nie2017multiclass,wang2022capped}, SVOREX tends to overly focus on outliers, neglecting true samples and thus significantly reducing performance. The capped $\ell_{p}$-norm mitigates this issue by imposing an upper bound on the ordinal hinge loss, eliminating larger residuals and helping the model mitigate the impact of outliers during training. Moreover, 
we also provide an intuitive explanation of why CSVOR is robust to outliers. CSVOR uses a binary matrix in each iteration to detect outliers with large residuals and then removes these outliers from the iteration to reduce their impact.
To the best of our knowledge, this is the first work that attempts to improve the robustness of SVOREX to outliers.

The capped $\ell_{p}$-norm ordinal hinge loss introduces non-convexity and increases the non-smoothness of the objective, making optimization very challenging.
To address this challenge, we introduce a Re-Weighted optimization framework to efficiently solve the proposed minimization problem.
The Re-Weighted algorithm is an efficient iterative algorithm tailored for solving certain non-convex problems.
Our theoretical results guarantee the convergence of the Re-Weighted optimization algorithm.

Our contributions can be summarized concisely in the following three aspects:
\begin{enumerate}
\item 
A novel Capped $\ell_{p}$-Norm Support Vector
Ordinal Regression(CSVOR) is proposed by 
 utilizing a capped $\ell_{p}$-Norm ordinal hinge loss function to eliminate data with large residuals during the training process.
Moreover, We explain from an intuitive perspective why CSVOR is robust to outliers.
 A  weight matrix $D$ is used to detect and
eliminate outliers  implicitly during the training process to resist the influence of outliers.
\item  A Re-weighted(RW) optimization framework is proposed to efficiently solve the minimization problem, and then it is  employed to solve the proposed capped $\ell_{p}$-norm problem.
Theoretical results ensure the convergence of the optimization algorithm.
\item
Extensive experiments on two artificial datasets and ten benchmark datasets show that our method outperforms state-of-the-art(SOTA) methods, especially in the presence of outliers.
Moreover, we conducted convergence experiments on the proposed algorithm, which showed that our algorithm can converge after a few iterations.
In addition, visualization experiments on the commonly used age estimation dataset FG-NET illustrate the potential of the proposed model in anomaly detection.

\end{enumerate}

The remainder of this paper is organized as follows. 
In Section 2, we briefly review ordinal regression and SVOR models. 
Section 3 presents the proposed CSVOR model utilizes the capped $\ell_{p}$-norm ordinal hinge loss. An Re-weighted algorithm for solving the optimization problem is introduced in Section 4.
In Section 5, we provide some theoretical analysis. 
In Section 6, we conduct extensive empirical studies on two synthetic datasets and several benchmark datasets to validate the effectiveness of our method.
Finally, Section 7 concludes this paper.
\section{Related work}
\subsection{Ordinal Regression}
Ordinal regression has been widely researched in recent years.
An excellent survey on ordinal regression is presented in \cite{gutierrez2015ordinal}, which categorizes existing models into three groups. 

The first group of methods directly transforms ordinal regression into either multi-class classification or metric regression. For example, Diaz and Amit \cite{diaz2019soft} propose a new encoding scheme that converts ordinal labels into soft probability distributions based on a distance metric. This encoding method pairs well with common classification loss functions such as cross-entropy.
Another method, the Moving Window Regression (MWR) method \cite{shin2022moving}, utilizes both local and global regressors to predict the class range of patterns separately. This approach often leads to more accurate predictions compared to other methods.

The second group, known as ordinal binary decompositions, leverages order information to break down the original ordinal regression into multiple binary classification tasks. Frank and Hall \cite{frank2001simple} propose that an ordinal regression problem with $K$ classes can be transformed into $K-1$ binary classification problems based on whether a pattern $\boldsymbol x$ is greater than $k$. In the work of Waegeman et al. \cite{waegeman2009ensemble}, different weights are applied to each sample of each binary classifier so that samples with predicted labels further away from the true label are more penalized. Perez-Orti{'z} et al.  \cite{perez2013projection} propose a new method for this decomposition. For intermediate subsets, labels are divided into three groups (less than $k$, equal to $k$, and greater than $k$). For extreme subsets, labels are divided into two categories. This incorporates sequential information into the sub-problems. A method for ordinal regression model by constructing proximal non-parallel hyperplanes is proposed based on the above decomposition method, which can obtain more flexible separate hyperplanes \cite{jiang2021non}.

The third group, referred to as threshold models, is grounded in the general concept of approximating a real-value predictor and subsequently partitioning the real line into intervals. Cumulative Link Models (CLMs) \cite{gouvert2020ordinal}, originated from a statistical background, are one of the first models designed specifically for ordinal regression problems. 
CLMs  use a special cumulative probability function to predict the probability of continuous classes. 
Vargas et al.  \cite{vargas2020cumulative} combine CLMs with the Continuous Quadratic Weighted Kappa (QWK) loss function to propose an ordinal deep network, successfully applying it to tasks such as age prediction and medical image classification.
The CORAL method \cite{shi2023deep} uses the last layer of the neural network to share weights to limit the posterior probability and ensure rank consistency.

There are also recent studies that do not fall into these three categories. For example, Liu et al. \cite{liu2020unimodal}  use unimodal constraints to reformulate the ordinal information in ordinal regression. Li et al. \cite{li2022unimodal}  introduce an efficient unimodal-concentrated loss. This loss maximizes the probability of an instance at its true value in a fully adaptive manner while ensuring a unimodal distribution.

\subsection{Support Vector Ordinal Regression(SVOR)}
Initially designed for binary classification problems, Support Vector Machines (SVMs) try to find an optimal direction to map feature vectors to function values on the real line and use a single optimization threshold to divide the real line into two regions.
Herbrich et al. \cite{shashua2002ranking}  propose the first ordinal regression algorithm based on SVM, which considers a pairwise approach and constructs a new dataset using difference vectors. 
 The work by Shashua and Levin \cite{shashua2002ranking} introduce two methods: maximizing the margin between the closest neighboring classes and maximizing the sum of margins between classes.
 Both approaches encounter two primary issues: the model is incompletely specified as the thresholds are not uniquely defined, and at the optimal solution, they may not be properly ordered, since the inequality $(b_1 \leq b_2 \leq \ldots \leq b_{K-1})$ is not included in the formulation. Consequently,
 Chu and Keerthi \cite{chu2005new,chu2007support} present two distinct reformulations of the same concept, addressing the issue of unordered thresholds in the solution.
  Li and Lin \cite{lin2012reduction} treat the ordinal regression as a special case of cost-sensitive classification and propose a reduction framework for ordinal regression. They then combine this framework and SVM to propose a new support vector ordinal regression model (REDSVM).
   In recent years, people have begun to pay attention to the issue of noise sensitivity in support vector learning for ordinal regression, where a small amount of noise can greatly impact the performance of models \cite{zhang2020towards}. 
   Zhong et al. \cite{zhong2023ordinal} combine the pinball loss with SVOR, introducing a novel support vector model for ordinal regression called PINSVOR. This approach mitigates the  sensitivity of SVOR to noise near the classification boundary.
\subsection{Briefly Review of Support Vector Learning for Ordinal Regression with  Explicit Constraints}
\begin{figure}[htbp]
    \centering
    \includegraphics[scale=0.5]{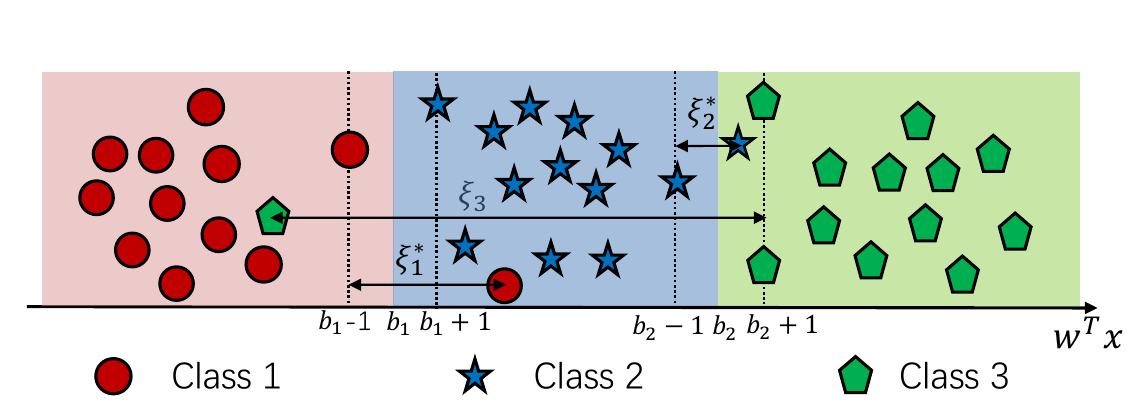} 
    \caption{An illustration of the definition of the loss $\xi$ and $\xi^*$. The samples from different classes,represented as patterns of different colors and shapes, are mapped by $\boldsymbol{w}^T\boldsymbol{x}$ onto the axis of function value.}
     \label{svorF}
    \end{figure} 
SVOREX is a widely used ordinal regression method that aims to classify data into $K$ ordered classes by finding $K-1$ parallel classification hyperplanes.
It includes a projection vector $\boldsymbol{w}$ and a threshold vector $\boldsymbol{b}=(b_0,b_1,...,b_{K})^T$ which satisfies the constraints $b_1 \le b_2 \le \dots \le b_{K-1}$ and $b_0=-\infty$, $b_K=+\infty$.
Similar to the binary classification, the projection direction maps feature vectors to values on the real line, and then the threshold vector divides the real line into $K-1$ regions, each corresponding to a class, as shown in Fig. \ref{svorF}.
For a data point $(\boldsymbol{x_i},y_i)$, its mapped value $\boldsymbol{w}^T\boldsymbol{x_i}$ should be assigned to the interval $[b_{y_i-1}+1, b_{y_i}-1]$.
If the mapped value is less than the lower bound $b_{y_i-1}+1$ is the error (denoted as $\xi(\boldsymbol{w},\boldsymbol{b}|\boldsymbol{x_i},y_i)$),  $b_{y_i-1}+1-\boldsymbol{w}^T\boldsymbol{x_i}$
and  whereas if the mapped value is greater than the upper bound, $\boldsymbol{w}^T\boldsymbol{x_i}-b_{y_i}-1$ is the error (denoted as $\xi^*(\boldsymbol{w},\boldsymbol{b}|\boldsymbol{x_i},y_i)$ ). 
The ordinal hinge loss can be written in the following form:

\begin{align}
\label{ordinalhingloss}
\xi_O(\boldsymbol{w},\boldsymbol{b}|\boldsymbol{x_i},y_i)=  \xi(\boldsymbol{w},\boldsymbol{b}|\boldsymbol{x_i},y_i)+\xi^*(\boldsymbol{w},\boldsymbol{b}|\boldsymbol{x_i},y_i),
\end{align}
where $\xi(\boldsymbol{w},\boldsymbol{b}|\boldsymbol{x_i},y_i)=\max(0,1-\boldsymbol{w}^T\boldsymbol{x_i}+b_{y_i-1})$ , $\xi^*(\boldsymbol{w},\boldsymbol{b}|\boldsymbol{x_i},y_i)=\max(0,1+\boldsymbol{w}^T\boldsymbol{x_i}-b_{y_i})$.
Similar to many SVM-based methods, SVOREX adopts the $l_2$ norm of $\boldsymbol{w}$ as the regularization term, which helps the model avoid overfitting.
The final optimization problem of SVOREX is as follows:

\begin{align}
\label{svorex}
\begin{aligned}
\min_{\boldsymbol{w},\boldsymbol{b}} & \sum_{i=1}^N\xi_O(\boldsymbol{w},\boldsymbol{b}|\boldsymbol{x_i},y_i)+\gamma||\boldsymbol{w}||_2^2\\
\text{s.t}.& \quad b_{k-1}\le b_k \quad k=1,2,\ldots K,
\end{aligned}
\end{align}
where $\gamma$ is regularization parameter.
Once the optimal values of $\boldsymbol{w}$ and $\boldsymbol{b}$ in problem (\ref{svorex}) are determined, the prediction rule is defined as follows:
\begin{align}
\label{decision ruler}
\begin{aligned}
r(\boldsymbol{x})&=\sum_{k=1}^{K-1}[[\boldsymbol{w}^T\boldsymbol{x}-b_k\ge 0]]+1,
\end{aligned}
\end{align}
 where $[[\cdot]]$ is an indicator function 
 the output is $1$ if the inner condition holds
 and $0$ otherwise.

\begin{figure}
    \centering
    \includegraphics[scale=0.3]{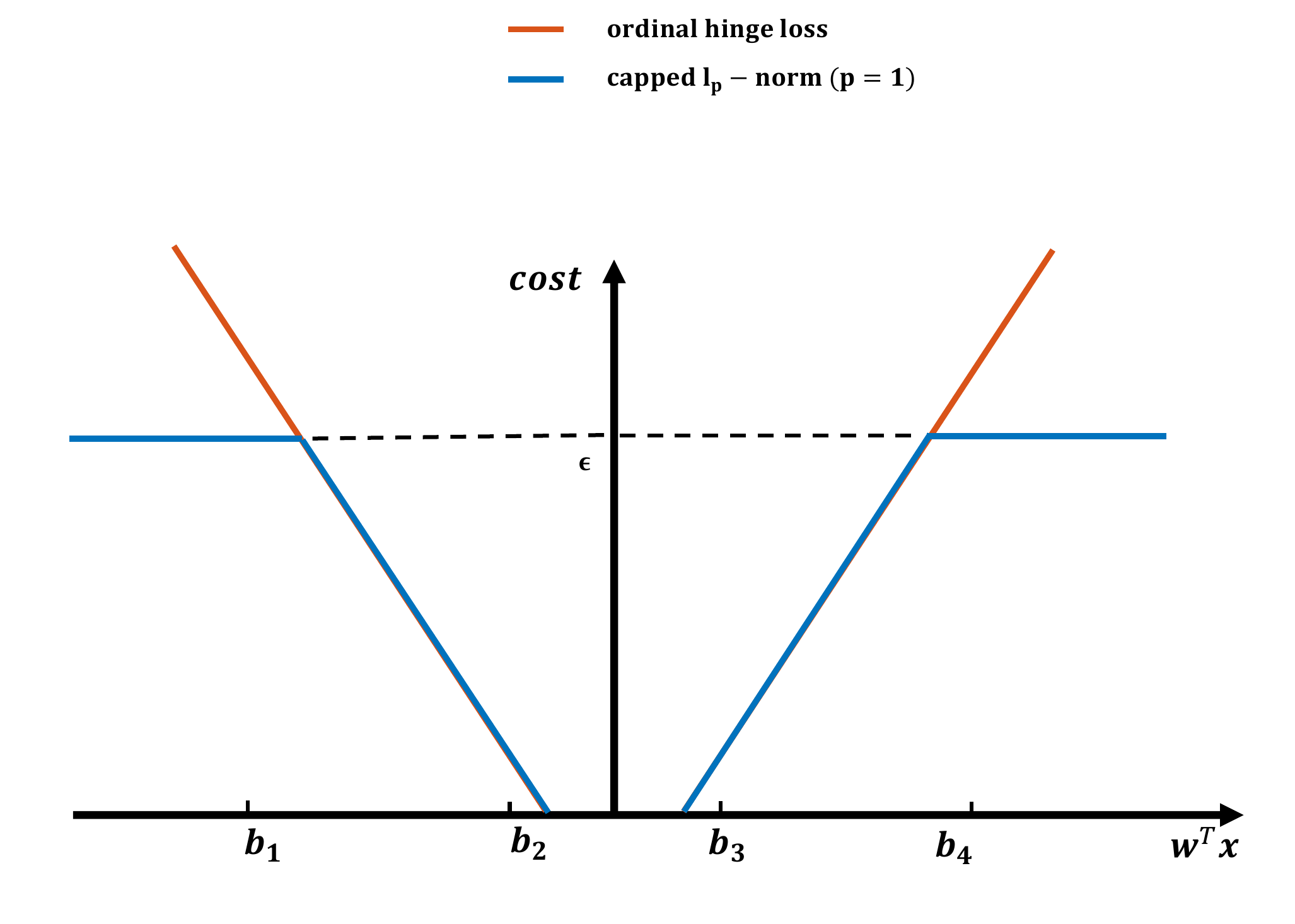}
    \caption{The plots of ordinal hinge loss and capped $\ell_{p}$ norm ordinal hinge loss.}
    \label{lossfunctionF}
\end{figure}

\section{Motivation and Proposed methodology}\label{motivation}

Despite its significant success,
the sensitivity of SVOREX to outliers can not be ignored.
Outliers usually have larger residuals and their presence affects model performance \cite{wang2022capped,nie2017multiclass}.
From Fig. \ref{lossfunctionF}, it can be observed that when a data point $\boldsymbol x$ is misclassified, the loss can grow infinitely large.  
Therefore, the ordinal hinge loss is not robust enough to data outliers.
To address this issue, we propose a capped $\ell_{p}$-norm ordinal hinge loss  as a robust and stable loss function to mitigate the impact of outliers.
We let $\boldsymbol{\xi_u}(\boldsymbol{w},\boldsymbol{b}|\boldsymbol{x_i},y_i)=(\xi(\boldsymbol{w},\boldsymbol{b}|\boldsymbol{x_i},y_i),\xi^*(\boldsymbol{w},\boldsymbol{b}|\boldsymbol{x_i},y_i))^T$.
The capped $\ell_{p}$-norm ordinal hinge loss function is defined as follows:
\begin{align}
\label{lp-norm-loss}
\xi_{c}(\boldsymbol{w},\boldsymbol{b}|\boldsymbol{x_i},y_i)=\min(||\boldsymbol{\xi_u}(\boldsymbol{w},\boldsymbol{b}|\boldsymbol{x_i},y_i)||_2^p,\epsilon),
\end{align}
where $0< p\le 2$ and $\epsilon $ determines the upper bound of the loss. 
We set $\epsilon \ge 1$, because a loss function with $\epsilon < 1$ becomes a constant before exiting the correct interval, it is impossible to distinguish between the correct interval and the wrong interval.
To illustrate this loss, we plot the loss function  for $p=1$ and $\epsilon=1$ in Fig. \ref{lossfunctionF}. 
We can observe that if data point $\boldsymbol{x}$ is misclassified, there will be a loss, but this loss will not exceed $\epsilon$.
This attribute is highly robust against data outliers because  no matter how large the residual loss is, it will never exceed $\epsilon$.
Like SVOREX, we also use the $\ell_{2}$-norm of the projected vector $\boldsymbol{w}$ as the regularization term. Therefore, we propose a new robust support vector ordinal regression method based on capped $\ell_{p}$-norm by solving the following optimization problem:

\begin{align}
\label{capped-lp-norm-svor}
\begin{aligned}
\min_{\boldsymbol{w},\boldsymbol{b}} & \sum_{i=1}^N \xi_{c}(\boldsymbol{w},\boldsymbol{b}|\boldsymbol{x_i})+\gamma||\boldsymbol{w}||_2^2\\
\text{s.t}.& \quad b_{k-1}\le b_k \quad k=1,2,\ldots K.
\end{aligned}
\end{align}

For a single hinge loss, it can be easily demonstrated that it can be written in the following form \cite{xiang2012discriminative}:
\begin{align}
\label{Convert}
\xi(\boldsymbol{w},b|\boldsymbol{x_i})=\min_{m_i\ge 0}|\boldsymbol{w}^T\boldsymbol{x_i}+b-y_i-y_im_i|.
\end{align}

Therefore, similar to Eq.(\ref{Convert}), the problem (\ref{capped-lp-norm-svor}) can be expressed in the following form:

\begin{align}
\label{optimization-problem}
\begin{aligned}
\min_{\boldsymbol{w},\boldsymbol{b},M\ge 0} & \sum_{i=1}^N \min(||(\boldsymbol{e}\boldsymbol{x_i}^T\boldsymbol{w}-\boldsymbol{\theta_i}-\boldsymbol{y^*}\circ \boldsymbol{m_i}||^p_2,\epsilon)+\gamma ||\boldsymbol{w}||_2^2
\\
\text{s.t}.& \quad b_{k-1}\le b_k \quad k=1,2,\ldots, K,
\end{aligned}
\end{align}
where $\boldsymbol{e}=(1,1)^T$, $\boldsymbol{\theta_i}=(b_{y_i-1}-1,b_{y_i}+1)^T$, $\boldsymbol{y^*}=(1,-1)^T$, and $M\in R^{2\times N}$ with the $i$-th column as $\boldsymbol{m_i}$. $M\ge 0$  means that each element of $M$ is not less than $0$.
$\circ$ denotes the hadamard product.
Due to the non-convex and non-smooth nature of the above problem, the optimization can be quite challenging.
In the next section we will introduce an efficient iterative algorithm to solve the optimization problem (\ref{optimization-problem}).

\section{Optimization Algorithm}
In this section, we will introduce an efficient optimization algorithm to optimize the problem (\ref{optimization-problem}).
\subsection{Algorithm to Solve a General Problem}
We first consider the following general problem:
\begin{align}
\label{re-weighted-optimization-problem}
\min_{x\in \Omega} h(x)+\sum_{i} v_i(g_i(x)),
\end{align}
where $v_i(x)$ is an arbitrary concave function with the domain of $g_i(x)$,
$x$ and $g_i(x)$ can be arbitrary scalars, vectors or matrices.
We propose an effective Re-weighted optimization framework to solve problem (\ref{re-weighted-optimization-problem}). 
The details of this optimization algorithm are described in Algorithm \ref{A1},
 where $f^{'}_i(g_i(x))$ denotes any supergradient of the concave function $f_i$ at point $g_i(x)$.
\begin{algorithm}[tb]
	\caption{Re-Weighted optimization framework }
    \label{A1}
    \begin{algorithmic} 
         \STATE $\textbf{Initialize} \enspace x \in \Omega \enspace \text{and} \enspace t=1$ 
     \WHILE{not converge}
     \STATE Compute the supergradient of the concave function $D_i=v_i^{'}(g_i(x))$ for each $i$.\\
     \STATE Update $x$ by the optimal solution to the problem: \\
     $\min_{x\in \Omega} h(x)+\sum_{i} Tr(D^T_ig_i(x))$
 \ENDWHILE
    \end{algorithmic}
\end{algorithm}

\subsection{Optimization Algorithm to Solve Problem (\ref{optimization-problem})}
We define the function $v_i(\cdot)$ as follows:
\begin{align}
\label{deqn_ex1a}
v_i(g_i(\boldsymbol{w},\boldsymbol{b},\boldsymbol{m_i}))=\min(g_i(\boldsymbol{w},\boldsymbol{b},\boldsymbol{m_i})^{\frac{p}{2}},\epsilon),
\end{align}
where $g_i(\boldsymbol{w},\boldsymbol{b},\boldsymbol{m_i})=||(\boldsymbol{w}^T\boldsymbol{x_i})\boldsymbol{e}+\boldsymbol{\theta_i}-\boldsymbol{y^*}\circ \boldsymbol{m_i}||^2_2$.

We can see that the problem (\ref{optimization-problem}) represents  a special case of problem (\ref{re-weighted-optimization-problem}) when $0<p\le 2$.
Therefore,
according to the second step of the Re-Weighted framework,  we solve the following optimization problem in each iteration:
\begin{align}
\label{problem_iter}
\begin{aligned}
&\min_{\substack{\boldsymbol{w},\boldsymbol{b},  M\ge 0}} \quad \sum_{i=1}^Nd_i||(\boldsymbol{e}\boldsymbol{x_i}^T\boldsymbol{w}-\boldsymbol{\theta_i}-\boldsymbol{y^*}\circ \boldsymbol{m_i}||^2_2+\gamma ||\boldsymbol{w}||_2^2\\
& \begin{array}{r@{\quad}l@{}l@{\quad}l}
s.t.& b_{k-1}\le b_k,\quad k=1,2,\ldots,K,\\\end{array} 
\end{aligned}
\end{align}
where
\begin{align}
\label{d_define}
d_i=\left\{ \begin{aligned} &\frac{p}{2}g_i(\boldsymbol{w},\boldsymbol{b},\boldsymbol{m_i})^{\frac{p-2}{2}} \quad & g_i(\boldsymbol{w},\boldsymbol{b},\boldsymbol{m_i})\le \epsilon \\ &0 & otherwise.\end{aligned} \right.
\end{align}
Note that the value of $d_i$ tends to infinity when $g_i(\boldsymbol{w},\boldsymbol{b},\boldsymbol{m_i})=0$ and $p< 2$.
We add an infinitesimal $\delta$ for the update of $d_i$ to avoid this situation.
More precisely, $d_i$ is update according to:
\begin{align}
\label{d_define1}
d_i=\left\{ \begin{aligned} &\frac{p}{2}(g_i(\boldsymbol{w},\boldsymbol{b},\boldsymbol{m_i})+\delta)^{\frac{p-2}{2}} \quad & g_i(\boldsymbol{w},\boldsymbol{b},\boldsymbol{m_i})\le \epsilon \\ &0 & otherwise.\end{aligned} \right.
\end{align}
Eq. (\ref{d_define1})  is equivalent to Eq. (\ref{d_define}) when $\delta \to 0$.

\subsubsection{update $\boldsymbol{w}$ and $\boldsymbol{b}$ when M is fixed}

When the $M$ is fixed, by setting the derivative of Eq. (\ref{problem_iter}) with respect to $\boldsymbol{w}$ to $0$, we have:
\begin{align}
\label{w_update}
\boldsymbol{w}=(2XDX^T+\gamma I)^{-1}XD\boldsymbol{e}^TH^T,
\end{align}
where $H\in \mathbb{R}^{2\times N}$ is a  matrix, each column of which is $\boldsymbol{\theta_i}+\boldsymbol{y^*}\circ \boldsymbol{m_i}$ and $I\in \mathbb{R}^{d\times d}$ is an identity matrix.
By substituting Eq. (\ref{w_update}) into problem (\ref{problem_iter}), we find that problem (\ref{problem_iter}) is a quadratic programming(QP) problem about $\boldsymbol{b}$.
There are many methods that can be used to solve this problem, such as
active set, 
conjugate gradient,
and 
interior point methods.

\subsubsection{update $M$ when  $\boldsymbol{w}$ and $\boldsymbol{b}$ are fixed}

When the $\boldsymbol{w}$ and $\boldsymbol{b}$ are fixed,
the problem (\ref{problem_iter}) can be solved by solving the following problem separately for each  $\boldsymbol{m_i}$:
\begin{align}
\label{mi_orignal}
\min_{\boldsymbol{m_i}\ge 0}||\boldsymbol{e}\boldsymbol{x_i}^T\boldsymbol{w}-\boldsymbol{\theta_i}-\boldsymbol{y^*}\circ \boldsymbol{m_i}||_2^2.
\end{align}
Note that $\boldsymbol{y^*}=(1,-1)^T$, the problem (\ref{mi_orignal}) is equivalent to the following form:
\begin{align}
\label{mi_finally}
\min_{\boldsymbol{m_i}\ge 0}||\boldsymbol{y^*}\circ(\boldsymbol{e}\boldsymbol{x_i}^T\boldsymbol{w})-\boldsymbol{y^*}\circ\boldsymbol{\theta_i}-\boldsymbol{m_i}||_2^2.
\end{align}

Since $\boldsymbol{m_i}\ge 0$, the optimal solution of problem (\ref{mi_finally}) can be easily obtained as follows:
\begin{align}
\label{M_update}
\boldsymbol{m_i}=(\boldsymbol{y^*}\circ(\boldsymbol{e}\boldsymbol{x_i}^T\boldsymbol{w})-\boldsymbol{y^*}\circ \boldsymbol{\theta_i})_+,
\end{align}
where the $i$-th element of vector $(\boldsymbol{u})_+$ is $\max(0,u_k)$.
Based on the above analysis, 
the detailed steps for solving the problem (\ref{optimization-problem}) are summarized in Algorithm \ref{A2}.

\begin{algorithm}[tb]
	\caption{Optimization algorithm to solve problem (\ref{optimization-problem})}
    \label{A2}
    \begin{algorithmic} 
    
    \STATE $\textbf{Input} \quad \text{the training dataset} \quad S=\{(\boldsymbol{x_i},y_i)\}_{i=1}^N$ \\
     \STATE $\textbf{Initialize} \quad  D=I \quad  \text{and} M=0$. 
     \WHILE{not converge}
     \STATE 1. Update $\boldsymbol{b}$ by solving the QP problem.\\
     \STATE 2. Update $\boldsymbol{w}$ by Eq.(\ref{w_update}):\\
     \quad $\boldsymbol{w}=(2XDX^T+\gamma I)^{-1}XD\boldsymbol{e}^TH^T$.
     \STATE 3. Update $M$, where the vector corresponding to column $i$ is caculated by Eq.(\ref{M_update}):\\
     \quad $\boldsymbol{m_i}=(\boldsymbol{y^*}\circ(\boldsymbol{e}\boldsymbol{x_i}^T\boldsymbol{w})-\boldsymbol{y^*}\circ \boldsymbol{\theta_i})_+$.
     \STATE 4. Update $D$  by Eq.(\ref{d_define}).
 \ENDWHILE
 
    \end{algorithmic}
\end{algorithm}

\section{THEORETICAL ANALYSIS}\label{1}
\subsection{Convergence Analysis}

In this subsection, we present the convergence analysis of Algorithm \ref{A1}. Since Algorithm \ref{A2} is a special case of Algorithm \ref{A1}, it is sufficient to  provide the convergence analysis of Algorithm \ref{A1}.

\begin{theorem}
The Algorithm \ref{A1} will reduce the objective value of problem (\ref{re-weighted-optimization-problem}) in each iteration until it  converges.
\end{theorem}

\begin{proof}
 Suppose the updated $x$ is $x^{new}$,  in accordance with the second step in Algorithm \ref{A1}, we have:

 \begin{align}
\label{h_update}
 h(x^{new})+\sum_{i} Tr(D^T_ig_i(x^{new}))\le h(x)+\sum_{i} Tr(D^T_ig_i(x)).
 \end{align}

Because $v_i(x)$ is concave for each $i$, according to the definition of supergradient, the following inequality holds:
  \begin{align}
\label{f_update}
v(g_i(x^{new}))-v(g_i(x))\le Tr(D^T_ig_i(x^{new}))-Tr(D^T_ig_i(x)).
 \end{align}

Thus, we have:
 \begin{align}
\label{all_update}
\begin{aligned}
&\sum_i v(g_i(x^{new}))-\sum_i  Tr(D^T_ig_i(x^{new})) \\ & \le 
\sum_i  v(g_i(x))-\sum_i  Tr(D^T_ig_i(x)).
\end{aligned}
 \end{align}

Combining the  two inequalities Eq. (\ref{h_update}) and Eq. (\ref{all_update}),
the following inequality holds:

 \begin{align}
\label{conclusion}
 h(x^{new})+\sum_i v(g_i(x^{new}))\le h(x)+\sum_i  v(g_i(x)).
 \end{align}
The equality in Eq.(\ref{conclusion}) holds only when the algorithm converges. Therefore, Algorithm 1 will monotonically decrease the objective of problem (\ref{re-weighted-optimization-problem}) at each iteration until convergence.
 \end{proof}

\subsection{An intuitive explanation of why CSVOR is robust to outliers}
In this subsection, an intuitive explanation is provided to illustrate why CSVOR is robust to outliers.

Like many SVM-based methods, the projection direction corresponding to SVOREX is only determined by a small part of the training data. These points play a crucial role in the model. The projection vector of SVOREX is determined by \cite{chu2007support}:
\begin{align}
\label{w_SVOREX}
\boldsymbol{w}=\frac{1}{2\gamma}\sum_{i=1}^N(\alpha_i-\alpha_i^*)\boldsymbol{x_i}.
\end{align}
The set of data points corresponding to $\alpha_i-\alpha_i^*\ne 0$ is $A=\{\boldsymbol{x_i}|\xi_i>0 \text{ and } \xi_i^*=0\}\cup\{\boldsymbol{x_i}|\xi^*_i>0  \text{ and } \xi_i=0\}\cup\{\boldsymbol{x_i}|\boldsymbol{w}^T\boldsymbol{x_i}-b_{y_i-1}-1=0 \text{ or } \boldsymbol{w}^T\boldsymbol{x_i}-b_{y_i}+1=0 \}$ \cite{chu2007support}.
Points outside $A$ have no effect on the calculation of the projection vector 
$\boldsymbol{w}$.
Since outliers usually have larger residuals, they generally belong to set $A$ and play an important role in the calculation of the projection direction $\boldsymbol{w}$.
However, according to Eq. (\ref{w_update}), it can be seen that in each iteration, the points with $g_i(\boldsymbol{w},\boldsymbol{b},\boldsymbol{m_i})>\epsilon$ do not play a role in the calculation of $\boldsymbol{w}$.
In other words, during the training process, CSVOR uses a  weight matrix $D$ to identify and 
remove outliers  implicitly so that they do not affect the calculation of the projection direction to resist the influence of outliers.
This can be seen as an intuitive explanation for why CSVOR is robust to outliers.


\section{Experiments}
In this section, extensive experiments are conducted on synthetic and benchmark datasets to validate the effectiveness of the proposed method.
\subsection{Experiments on Synthetic Datasets}
We illustrate the robustness of our algorithm against outliers using two synthetic datasets. The first dataset consists of randomly generated two-dimensional data with three Gaussian clusters, each represented by different colors and shapes. Specifically, the red circles represent samples of the first class, the orange triangles represent samples of the second class, and the blue squares represent samples of the third class.
We compare our CSVOR algorithm with SVORIM and SVOREX, and the comparison results are presented in Fig. \ref{first synthetic}. This figure displays the projection directions of our CSVOR algorithm and the compared algorithms as the number of outliers varies.
From Fig. \ref{f1}, we observe that when there are no outliers, all three algorithms can find a relatively good projection direction. However, when we introduce five outliers into the data, we observe that as the size of the outliers increases, SVOREX and SVORIM become inefficient. In contrast, our CSVOR model consistently performs well, demonstrating its robustness to outliers.
This example illustrates that our CSVOR model can effectively avoid the influence of outliers, making it robust to outliers.

The second synthetic dataset also consists of randomly generated two-dimensional data from three Gaussian clusters. We compare our CSVOR algorithm with SVORIM and SVOREX, as shown in Fig. \ref{second synthetic}. When there are no outliers, all three algorithms perform well.
To assess the robustness of these algorithms to outliers, we added 20 and 60 outliers to the data. From Fig. \ref{s2}, it can be seen that with 20 outliers, SVORIM becomes inefficient, and although SVOREX is not inefficient, the projection direction it found has deviated from the optimal projection direction to a certain extent. With 60 outliers, SVOREX becomes completely inefficient. In contrast, our CSVOR model continues to perform well.
This example illustrates the robustness of our CSVOR model to outliers as their number increases. As the number of outliers increases, the projection directions found by SVOREX and SVORIM deviate further from the optimal projection direction, while the projection direction found by our model remains close to the optimal one.

\begin{figure*}[htbp]
  \centering
  \subfigure[non-outlier]{
    \includegraphics[width=0.3\textwidth]{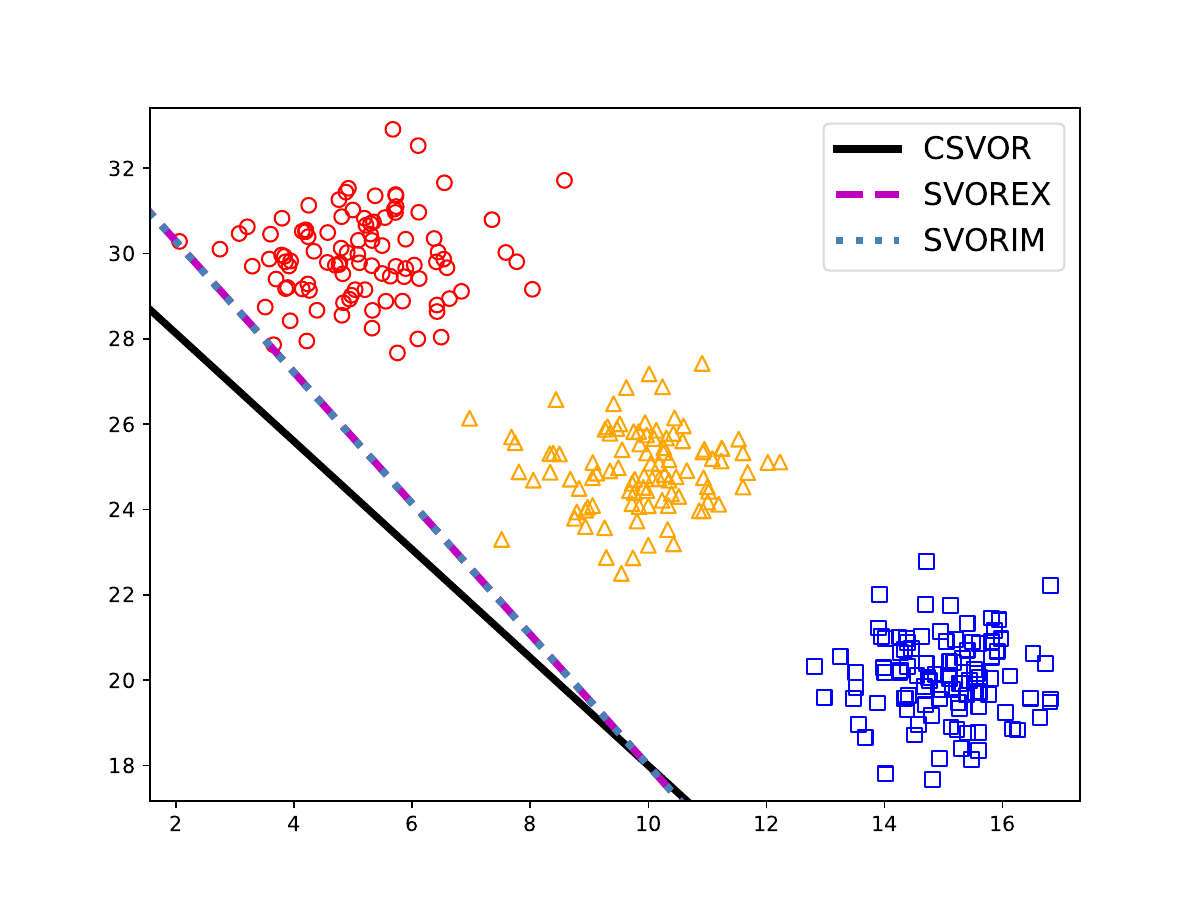}
    \label{f1}
  }
  \subfigure[The value of outliers is relatively large ]{
    \includegraphics[width=0.3\textwidth]{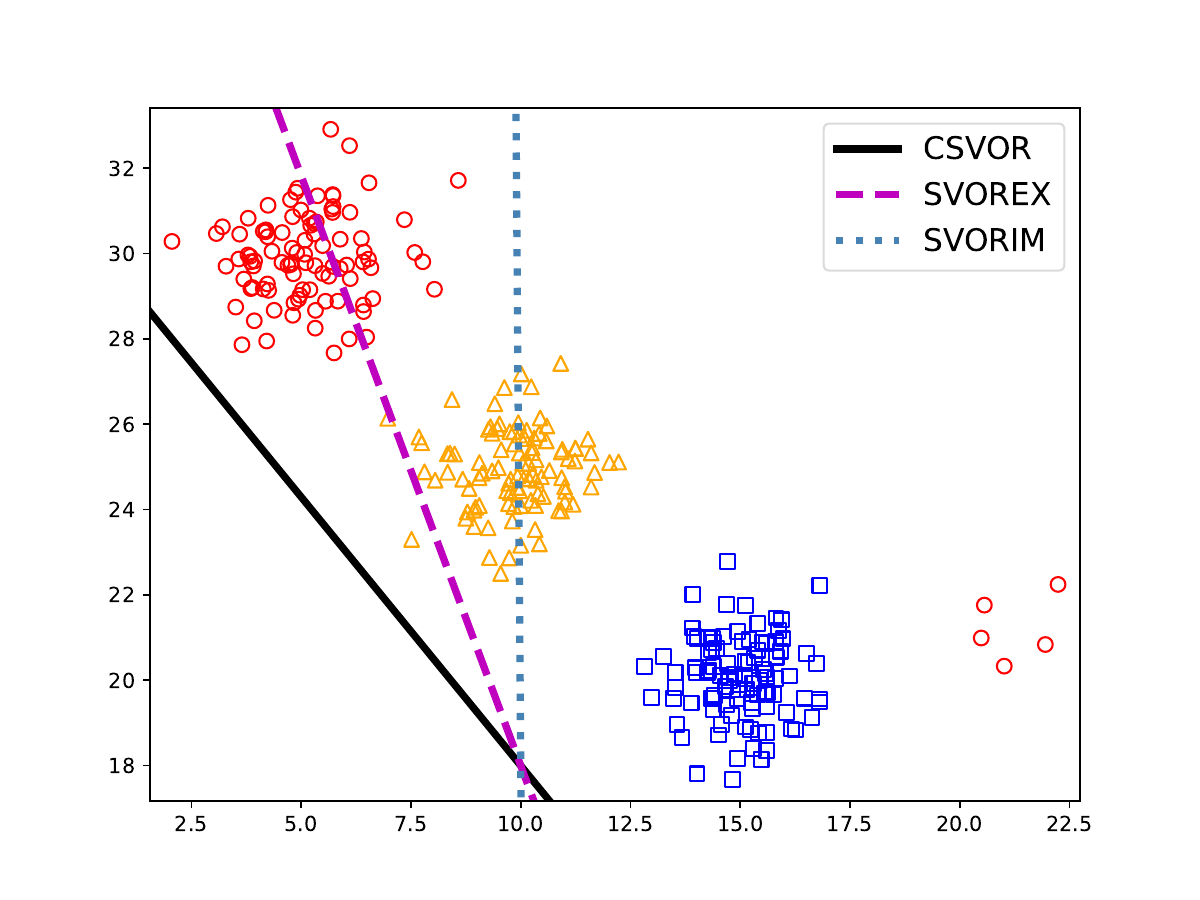}
    \label{f2}
  }
  \subfigure[The value of outliers is large ]{
    \includegraphics[width=0.3\textwidth]{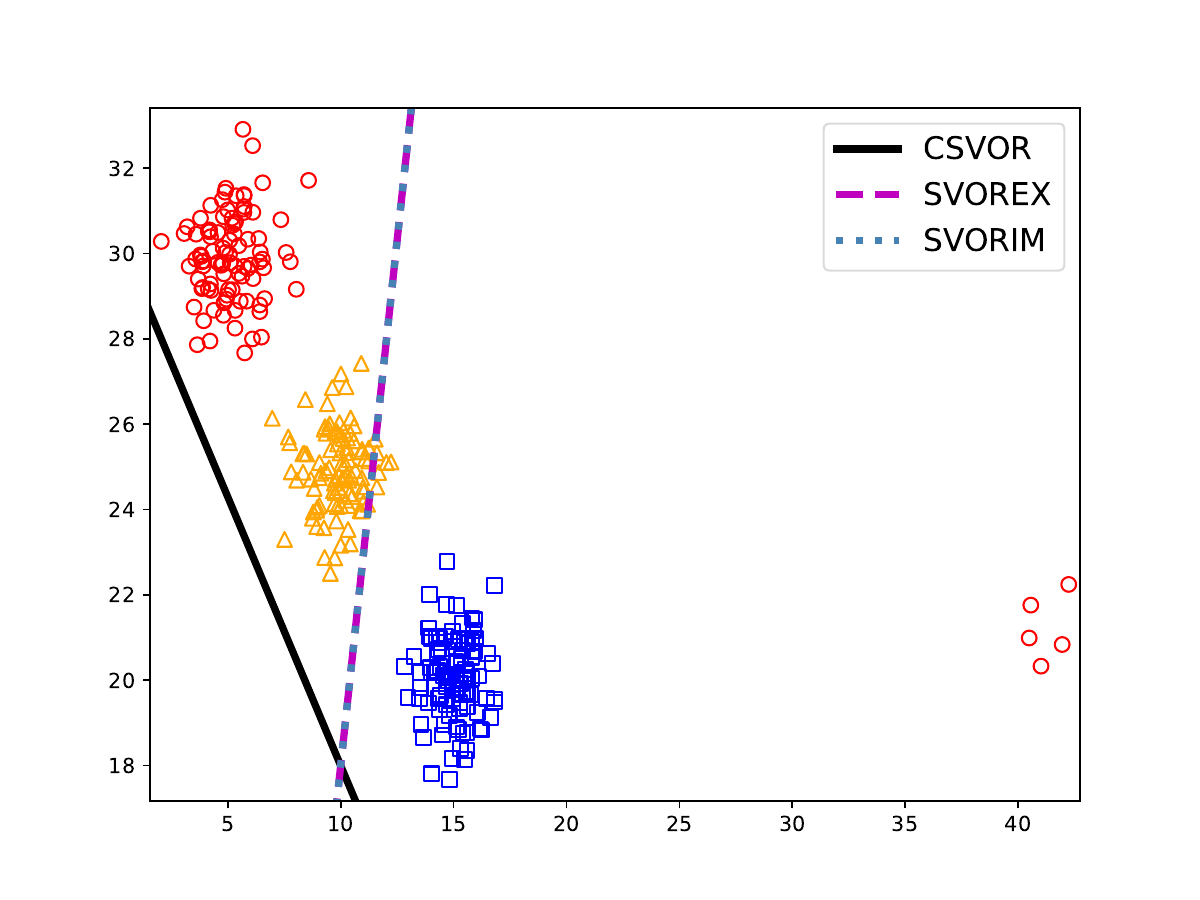}
    \label{f3}
  }
  
  \caption{Results on the first synthetic dataset.}
  \label{first synthetic}
\end{figure*}

\begin{figure*}[htbp]
  \centering
  \subfigure[non-outlier]{
    \includegraphics[width=0.3\textwidth]{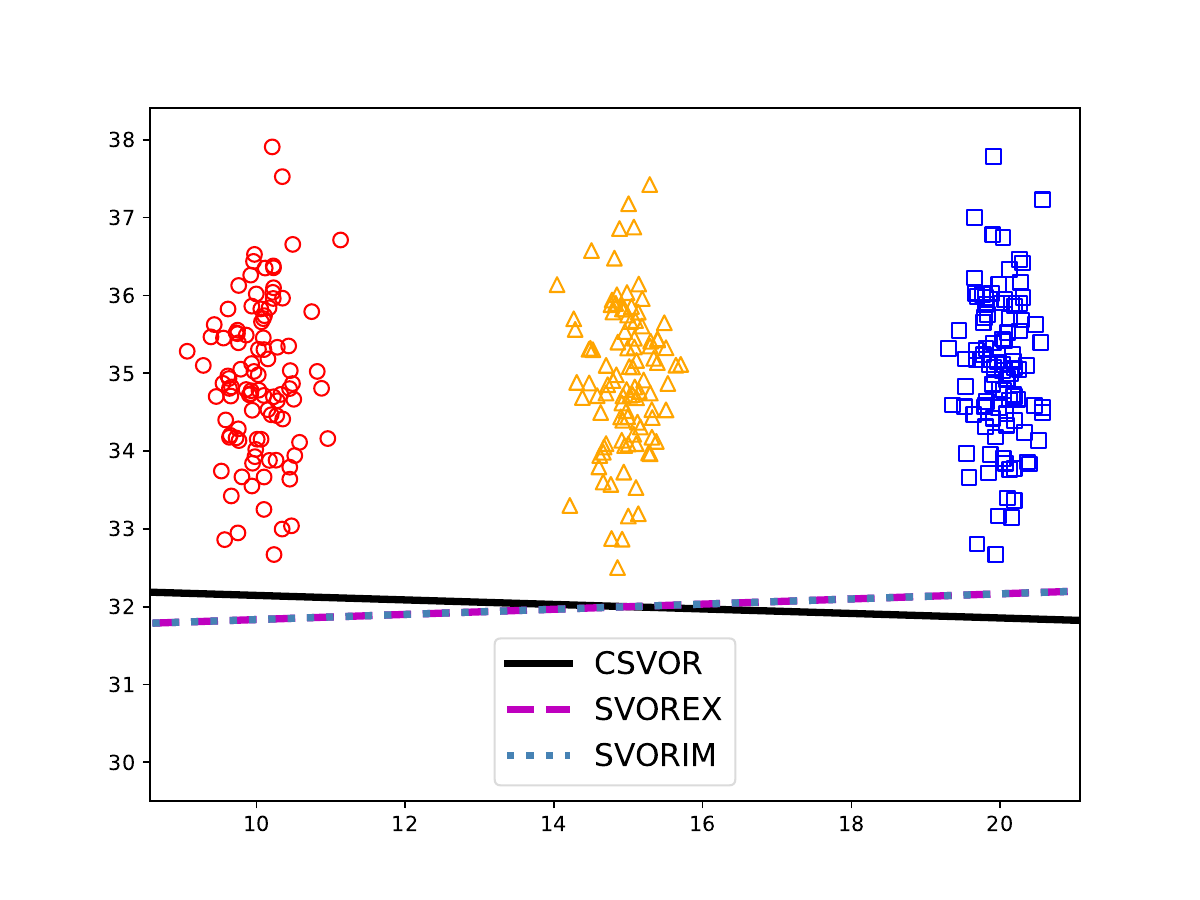}
    \label{s1}
  }
  \subfigure[The number of ouliers is 20 ]{
    \includegraphics[width=0.3\textwidth]{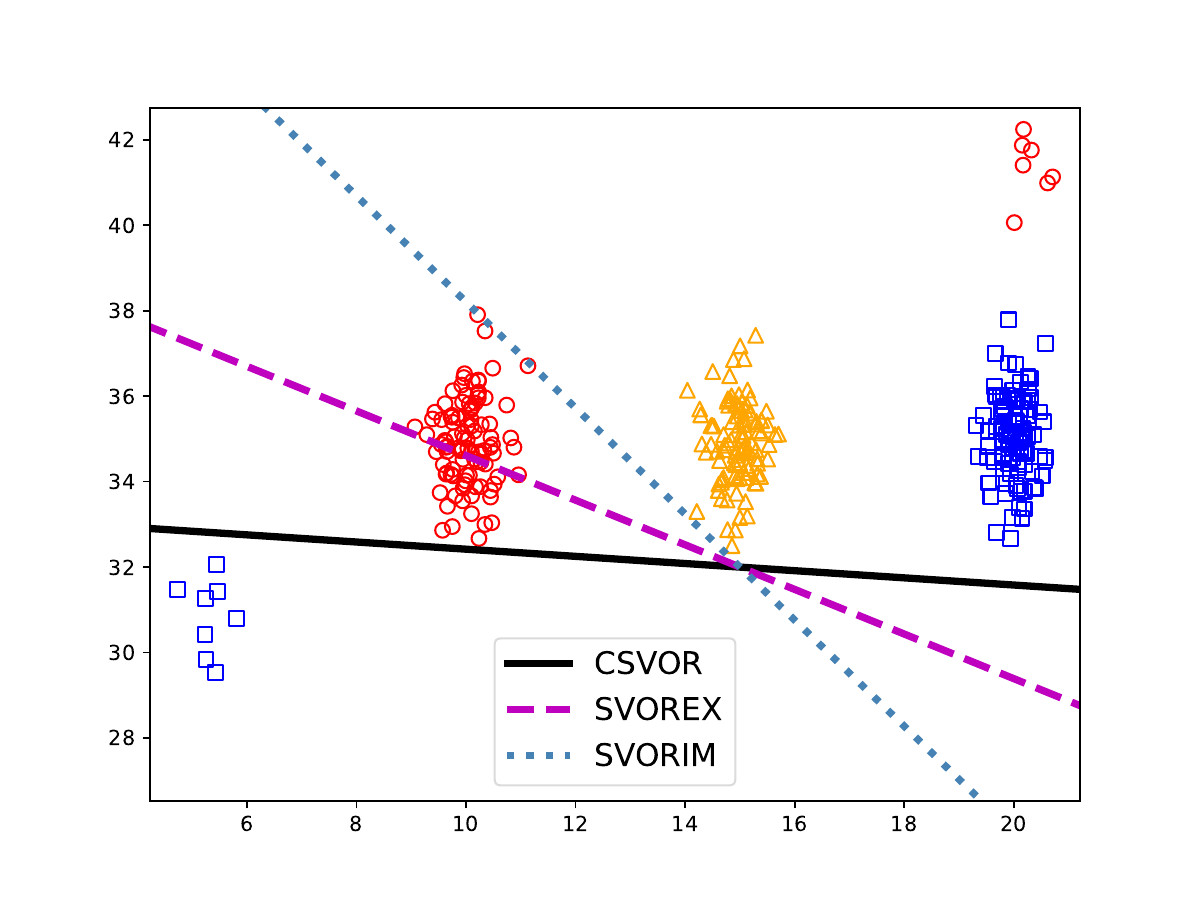}
    \label{s2}
  }
  \subfigure[The number of ouliers is 60]{
    \includegraphics[width=0.3\textwidth]{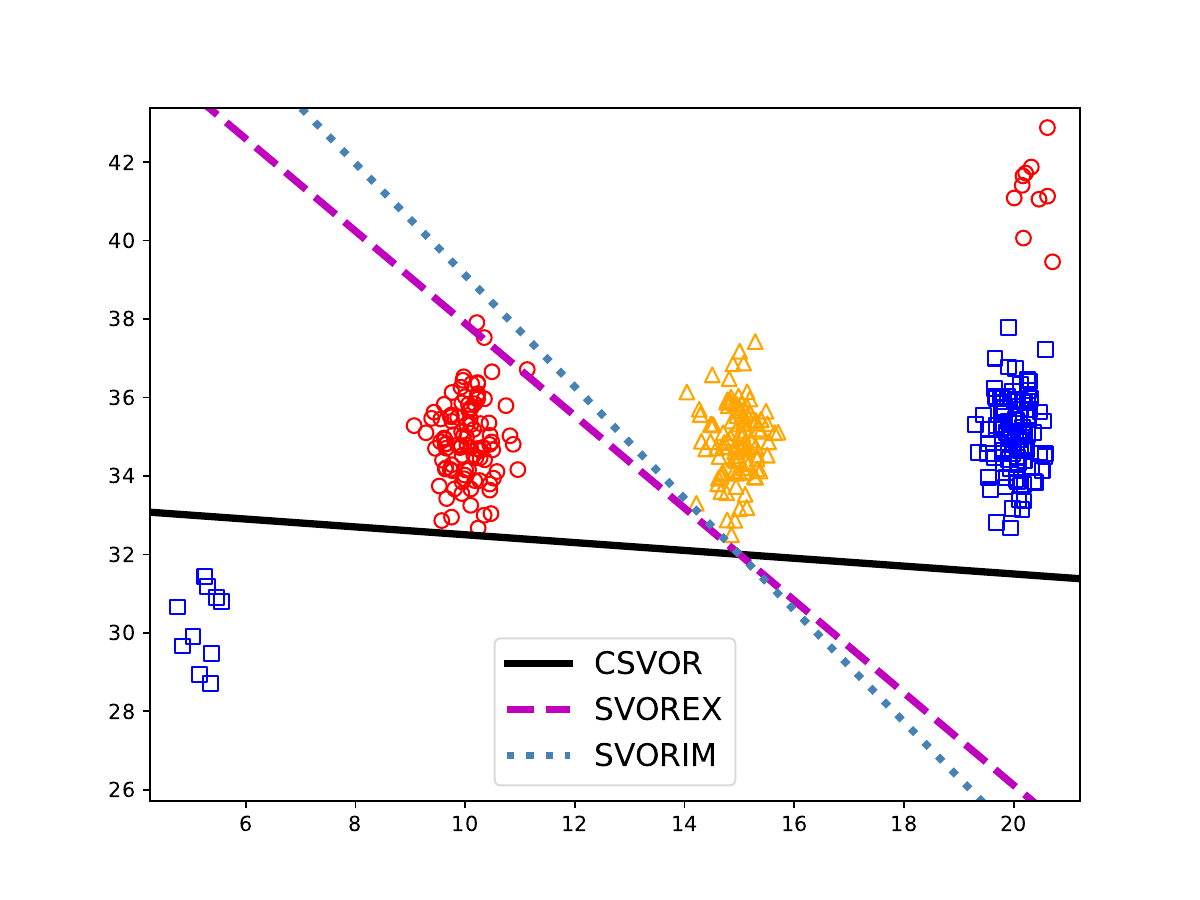}
    \label{s3}
  }
  
  \caption{Results on the second synthetic dataset.}
  \label{second synthetic}
\end{figure*}

\subsection{Experiments on Benchmark Datasets}
\subsubsection{Experimental Setting}
In this subsection, we conduct experiments on several benchmark datasets to compare the performance of CSVOR and state-of-the-art(SOTA)  methods.
The SOTA methods include SVOREX \cite{chu2007support},  SVORIM \cite{chu2007support}, PINSVOR \cite{zhong2023ordinal}, REDSVM \cite{lin2012reduction}, POM \cite{mccullagh1980regression}, KDLOR \cite{sun2009kernel}.
We use a total of 10 benchmark datasets, comprising 4 real datasets and 6 discretized regression datasets.

The 4 real datasets are commonly used UCI datasets in the context of ordinal regression \cite{sanchez2019orca}. The 6 discretized regression datasets, originally provided by Chu et al.  \cite{chu2007support}, are transformed into ordinal regression datasets by dividing the data into different classes with the same frequency.
The details of the datasets are in Table \ref{datasets}.
In order to compare the robustness to outliers, we add two different outliers to the real datasets and the discretised regression datasets respectively.

\begin{table}[htpb]
\caption{Dataset descriptions. \label{tab:table1}}
\centering
\label{datasets}
\begin{tabular}{cccc}
   \toprule
   Datasets& Instances & Attributes & Classes\\
   \midrule
Contact-Lenses & 24 & 6 & 3 \\
Squash-Unstored&52&52&3 \\
Tae &151&54&3\\
Newthyroid &215&5&3\\
Machine&209&7&5\\
Housing&506&14&5\\
Stock&700&9&5\\
Bank &8192&8&5\\
Computer &8192&12&5\\
Cal.housing &20640&8&5\\
   \bottomrule
\end{tabular}
\end{table}

We introduce label noise to the 4 real datasets to compare the robustness of our method and the baselines.
Label noise, considered a type of outlier, is the process that pollutes labels \cite{liu2022multi,frenay2013classification}.
For the real datasets, we consider 30 random stratified splits, allocating 75\% of the patterns to the training set and 25\% to the test set. Outliers are generated in the training set by randomly changing the labels of 1\%, 5\%, 10\%, 15\%, and 20\% of the samples.
Regarding the discrete regression datasets, which are transformed into ordinal regression by dividing the targets of traditional regression datasets into five different bins with equal frequency, we perform 20 random splits. The number of training and test instances follows the suggestions by Chu et al. \cite{chu2005new}. Outliers are generated by normalizing the data to [0, 1] and then randomly setting 50\% of the features of 1\%, 5\%, 10\%, 15\%, and 20\% of the data points to -1.

In terms of parameter settings, for SVOR methods (SVOREX, SVORIM, PINSVOR, and REDSVM), the regularization parameter $C$ is selected from the range $[10^{-3}, 10^{-2}, ..., 10^{3}]$.
The pinball loss parameter  $\tau$ in PINSVOR is chosen from $[0.1,0.2,\ldots, 0.9]$. The penalty coefficient $C$ in KDLOR is selected from $[10^{-1},0,10]$.
In our approach, 
the  parameters $\gamma$ and $p$ are selected from
within the specified ranges of $[0.5,0.7,1,1.5,1.7,2]$ and $[10^{-3},10^{-2},...,10^3]$, respectively.
Optimal values for all parameters are obtained using five fold cross-validation.
Additionally, we set the value of $\epsilon$ heuristically by selecting the top 10\% of data with the highest residuals in the first 5 iterations.

 The Mean Zero Error (MZE) is selected as the evaluation index, which reflects the accuracy of model prediction. The smaller the MZE on the test data set, the better the generalization performance of the ordered regression model. The calculation formula of MZE is as follows:
\begin{align}
\label{1.30}
MZE=\frac{1}{N}\sum_{i=1}^N[[y_i\ne y_i^{pred}]],
\end{align}
where $y_i$ is the true label and  $y_i^{pred}$ is  the predicted label. 
We report the performance of CSVOR and 7 baselines on 10 benchmark datasets in 
Table \ref{labelnoise} and Table \ref{outlier}. 
Fig. \ref{CDF} shows the critical difference (CD) plot of the Friedman test and Nemenyi post-hoc test at a 5\% significance level. The Friedman test with the Nemenyi post-hoc test is a frequently used non-parametric method for comparing the performance of multiple methods across a variety of datasets \cite{demvsar2006statistical,xu2017robust}. In the CD plot, a smaller average rank (abscissa) indicates better model performance. If the abscissa spacing between two methods is greater than the CD value , we consider there to be a significant difference in their performance. In Fig. \ref{CDF}, we connect the methods that show no significant difference with a red line.

\begin{table*}[htbp]
\centering
\caption{Test results for each real dataset and method(average MZE $\pm$  standard deviation ). Best results are in bold, and the second best results are underlined. (r is the ratio of label noise).}
\footnotesize
\resizebox{\textwidth}{!}{
\begin{tabular}{ccccccccc}
\toprule
Datasets & r(\%) & SVOREX & SVORIM & PINSVOR & REDSVM & POM & KDLOR & CSVOR \\ 
\midrule
\multirow{6}{*}{Contact-Lenses}
&0& \underline{$0.294\pm 0.113$}&
$0.339\pm 0.108$&$0.300\pm0.134$&$0.339\pm 0.120$&$0.394\pm 0.178$&	$0.350\pm 0.160$	& $\bf{0.222\pm 0.126}$\\
& 1   & $0.322 \pm 0.107$ & $0.356\pm0.129$ & \underline{$0.311\pm0.137$} & $0.367\pm0.068$ & $0.406\pm0.143$ & $0.344\pm0.190 $& $ \bf {0.278\pm0.126} $\\
                    & 5   & $0.350\pm0.166 $& $0.367\pm0.068$ & \underline{$0.333\pm0.152$} & $0.356\pm0.068$ & $0.406\pm0.189$ & $0.483\pm0.202$ &  $\bf{0.300\pm0.179}$ \\
                    & 10  & $0.361\pm0.077 $& $0.422\pm0.179$ & \underline{$0.344\pm0.115$} & $0.367\pm0.068$ & $0.444\pm0.166 $& $0.500\pm 0.196$ &  $\bf{0.333\pm 0.214}$ \\
                    & 15  & $0.367\pm 0.077$ & $0.439\pm0.198$ & $0.367\pm0.177$ & \underline{$0.367\pm0.068$} & $0.433\pm0.166$ & $0.456\pm0.169$ &  $\bf{0.356\pm0.206}$ \\
                    & 20  & $0.378\pm0.075$ & $0.472\pm0.219$ & \bf $0.367\pm0.132$ & $0.406\pm0.150$ & $0.444\pm 0.177 $& $0.517\pm0.187 $&  $\bf{0.367\pm0.159}$ \\
\midrule
\multirow{6}{*}{Squash-Unstored} 
&0&$0.226\pm 0.107$&$0.226\pm 0.107$&$\bf {0.208\pm 0.109}$&		$0.282\pm0.135$&$0.538\pm0.000$&$0.238\pm0.128$&\underline{$0.218\pm0.123$}\\
& 1   & $\bf {0.221\pm0.110}$ &$0.223\pm 0.109$ & $0.233\pm0.138$ & $0.346\pm0.112$ & $0.538\pm0.000$ & $0.300\pm0.133$ & \underline{ $0.221\pm0.112$} \\
                    & 5   & \underline{$0.236\pm0.103$} & $0.238\pm0.126 $& $0.244\pm0.125$ & $0.313\pm0.132$ & $0.538\pm0.000$ & $0.303\pm0.138$ & $\bf {0.223\pm0.095}$ \\
                    & 10  & \underline{$0.228\pm0.103$} & $0.236\pm0.116$ & $0.295\pm0.134$ & $0.326\pm0.135$ & $0.538\pm0.000 $& $0.369\pm0.143$ & $\bf 0.215\pm0.116$ \\
                    & 15  & \underline{$0.233\pm0.117$} & $0.267\pm0.116$ & $0.290\pm0.138$ & $0.349\pm0.145$ & $0.538\pm0.000$ & $0.472\pm0.104$ & $\bf {0.226\pm0.120}$ \\
                    & 20  & \underline{$0.246\pm0.102$} & $0.308\pm0.137$ &$ 0.326\pm0.155$ & $0.446\pm0.167$ & $0.538\pm0.000$ & $0.474\pm0.162$ & $\bf {0.244\pm0.102}$ \\
\midrule
\multirow{6}{*}{Tae} 
&0& \underline{$0.443\pm 0.053$}&	$0.443\pm 0.056$&		$0.468\pm 0.098$&$0.462\pm 0.077$&$0.642\pm 0.086$&$0.497\pm 0.080$	
&$\bf {0.432\pm 0.058}$\\
& 1   & \underline{$0.451\pm 0.059$} & $0.454\pm 0.057 $& $0.470\pm 0.087$ & $0.458\pm 0.067$ & $0.646\pm 0.078$ & $0.504\pm 0.075$ & $\bf {0.439\pm 0.057}$ \\
                    & 5   & \underline{$0.458\pm 0.062$} & $0.466\pm 0.065$ & $0.470\pm 0.083$ & $0.485\pm 0.071$ & $0.645\pm 0.078$ & $0.515\pm 0.080$ &$ \bf {0.444\pm 0.078}$ \\
                    & 10  & $0.500\pm 0.076$ & $0.507\pm 0.080$ & \bf{$0.481\pm 0.081$} & $0.529\pm 0.083$ & $0.649\pm 0.065$ & $0.547\pm 0.090$ & \underline{$0.486\pm 0.086$} \\
                    & 15  & $0.489\pm 0.070$ & $0.490\pm 0.070$& \underline{$0.468\pm 0.092$} & $0.514\pm 0.082$ & $0.649\pm 0.066$ & $0.527\pm 0.095$ & $\bf 0.469\pm 0.073$ \\
                    & 20  & $0.523\pm 0.060$ & $0.529\pm 0.057$ & $0.492\pm 0.086$ & $0.547\pm 0.073$ &$ 0.651\pm 0.059$ & $0.570\pm 0.089$ & $\bf  { 0.518\pm 0.056} $\\
\midrule

\multirow{6}{*}{Newthyroid}
&0&\underline{$0.031\pm 0.023$}	&\underline{$0.031\pm 0.023$}&$0.052\pm 0.023$&$\bf{0.027\pm 0.019}$&$0.035\pm 0.023$&$0.047\pm 0.026$ &$0.037\pm 0.031$\\
& 1   & $\bf 0.031\pm 0.025$ & $0.036\pm 0.030$ & $0.054\pm 0.033$ & \underline{$0.033\pm 0.024$} & $0.036\pm 0.026$ & $0.049\pm 0.027 $&  $0.041\pm 0.027$ \\
                    & 5   & \underline{ $0.043\pm 0.033$} & $0.056\pm 0.042 $& $0.051\pm 0.033$ & $0.054\pm 0.037$ & $0.054\pm 0.035$ & $\bf 0.040\pm 0.024$ &  $0.043\pm 0.028$ \\
                    & 10  & $0.067\pm 0.036$& $0.082\pm 0.037$ & \underline{$0.049\pm 0.025$} & $0.075\pm 0.037$ & $0.077\pm 0.040$& $0.049\pm 0.047$ & $\bf{0.038\pm 0.021}$ \\
                    & 15  &$0.086\pm 0.043$ & $0.104\pm 0.052$ & \underline{$0.049\pm 0.030$} & $0.123\pm 0.049 $& $0.106\pm 0.052$ & $0.077\pm 0.053$ & $\bf{0.048\pm 0.026}$ \\
                    & 20  & $0.111\pm 0.037$ & $0.132\pm 0.044$ & \underline{$0.049\pm 0.028$} & $0.169\pm 0.057$ & $0.138\pm 0.046$ & $0.104\pm 0.081$ &  $\bf{0.038\pm 0.023}$ \\
\bottomrule
\end{tabular}}
\label{labelnoise}
\end{table*}

\begin{table*}[htbp]
\centering
\caption{Test results for each discretised regression dataset and method(average MZE $\pm$  standard deviation ). Best results are in bold, and the second best results are underlined. (r is the ratio of data outliers).}
\footnotesize
\resizebox{\textwidth}{!}{
\begin{tabular}{ccccccccc}
\toprule
Datasets & r(\%) & SVOREX & SVORIM & PINSVOR & REDSVM & POM & KDLOR & CSVOR \\ 
\midrule
\multirow{6}{*}{Machine} 
&0&
$0.406\pm 0.058$&  $0.405\pm 0.060$&$0.399 \pm 0.076$&\underline{$0.391\pm 0.068$}&	$0.397\pm 0.068$&$0.428\pm 0.064$& $\bf{0.388\pm 0.067}$\\
& 1   &$ 0.407\pm0.063$ & $0.418\pm0.066$ & $0.406\pm0.061$ & \underline{$0.402\pm0.070$} & $0.415\pm0.062$ & $0.458\pm0.063$ & $\bf {0.393\pm0.060}$ \\
                    & 5   & \underline{$0.418\pm0.056$}& $0.449\pm 0.071$ & $0.425\pm 0.055$ & $0.419\pm 0.066$ & $0.442\pm 0.055 $& $0.471\pm 0.062$ & $\bf{0.413\pm 0.063 }$ \\
                    & 10  & \underline{$0.436\pm 0.064$} & $0.524\pm 0.061$ & $0.457\pm 0.071$ & $0.479\pm 0.078$ & $0.455\pm 0.064$ & $0.498\pm 0.056$ & $\bf{0.426\pm 0.054}$ \\
                    & 15  & $0.484\pm 0.086$ & $0.600\pm 0.103$ & $0.463\pm 0.074$ & $0.508\pm 0.096$ & \underline{$0.460\pm 0.088$} & $0.494\pm 0.079$ & $\bf {0.436\pm 0.067}$ \\
                    & 20  & $0.531\pm 0.081$ & $0.649\pm 0.103 $& $0.487\pm 0.072$&$ 0.581\pm 0.110$ & \underline{$0.460\pm 0.049$} & $0.503\pm 0.083$ &$ \bf{0.436\pm 0.062}$ \\
\midrule
\multirow{6}{*}{Housing} &0&$0.350\pm 0.019$	&$0.357\pm 0.019$	&$\bf{0.285\pm 0.033}$&$0.356\pm 0.024$&	$0.353\pm 0.020$&		$0.467\pm 0.040$&\underline{$0.323\pm 0.033$}\\
& 1   & $0.356\pm 0.028$ & $0.366\pm 0.030$ & $0.383\pm 0.029$ & $0.367\pm 0.027$ & \underline{$0.355\pm 0.062$} & $0.456\pm 0.039$ & $\bf {0.344\pm 0.033}$ \\
                    & 5   & $0.371\pm 0.037$ & $0.401\pm 0.043$ & $0.410\pm 0.039$ & $0.396\pm 0.041$ & \underline{$0.386\pm 0.055$} & $0.493\pm 0.062$ & $\bf {0.351\pm 0.040}$ \\
                    & 10  & \underline{$0.401\pm 0.045$} & $0.449\pm 0.045$ & $0.443\pm 0.048$ & $0.438\pm 0.044$ & $0.421\pm 0.064$ & $0.507\pm 0.041$ &  $\bf{0.381\pm 0.044}$ \\
                    & 15  & $0.455\pm 0.028$ & $0.500\pm 0.046$ & $0.454\pm 0.028$ & $0.495\pm 0.047$ & \underline{$0.454\pm 0.088$} & $0.505\pm 0.042$ & $\bf{ 0.409\pm 0.046}$ \\
                    & 20  & $0.485\pm 0.044$ & $0.535\pm 0.047$ & \underline{$0.466\pm 0.034$} & $0.526\pm 0.051$ & $0.477\pm 0.049$ & $0.525\pm 0.046$ & $\bf{0.425\pm 0.047}$ \\
\midrule
\multirow{6}{*}{Stock} 
&0&\underline{$0.357\pm 0.015$}&$0.367\pm 0.017$&$0.358\pm 0.095$&	$0.367\pm 0.018$&$0.365\pm 0.019$&$0.413\pm 0.022$&$\bf{0.323\pm 0.017}$	\\
& 1   & \underline{$0.363\pm 0.019$} & $0.390\pm 0.022$ & $0.410\pm 0.021$ & $0.390\pm 0.021$ & $0.384\pm 0.023$ & $0.475\pm 0.036 $& $\bf {0.346\pm 0.021}$ \\
                    & 5   & \underline{$0.408\pm 0.025$} & $0.412\pm 0.033$ & $0.438\pm 0.042$ & $0.409\pm 0.025$ & $0.445\pm 0.028$ & $0.501\pm 0.041$ & $\bf {0.365\pm 0.019}$ \\
                    & 10  & \underline{$0.439\pm 0.041$} & $0.456\pm 0.033 $& $0.473\pm 0.030$ & $0.456\pm 0.050$& $0.473\pm 0.028$ & $0.515\pm 0.033$ & $\bf{0.390\pm 0.030}$ \\
                    & 15  & \underline{$0.476\pm 0.026$} & $0.541\pm 0.074$ & $0.500\pm 0.031$ & $0.557\pm 0.086$ & $0.494\pm 0.028$ & $0.532\pm 0.032 $& $\bf{0.428\pm 0.047}$ \\
                    & 20  & \underline{$0.500\pm 0.030$} & $0.632\pm 0.063$ & $0.506\pm 0.047$& $0.681\pm 0.055$ & $0.503\pm 0.034$ & $0.531\pm 0.043$ & $\bf {0.455\pm 0.057}$ \\
\midrule
\multirow{6}{*}{Bank}
&0&\underline{$0.278\pm 0.026$}&\underline{$0.278\pm 0.026$}& $0.376\pm 0.046$&$0.280\pm 0.033$&$\bf {0.277\pm 0.025}$&$0.317\pm 0.064$&$0.284\pm 0.109$
\\
& 1   & \underline{$0.329\pm 0.108$} &$ 0.360\pm 0.158$ & $0.378\pm 0.078$ & $0.341\pm 0.141$ & $0.340\pm 0.129$ & $0.465\pm 0.140$ & $\bf{0.307\pm 0.109}$ \\
                    & 5   & \underline{$0.514\pm 0.109$} & $0.575\pm 0.158$ & $0.543\pm 0.091$ & $0.532\pm 0.098$ & $0.534\pm 0.095$ & $0.661\pm 0.061$& $\bf {0.399\pm 0.147}$ \\
                    & 10  & $0.616\pm 0.109$ & $0.627\pm 0.077$ & $0.609\pm 0.097$  & $\bf{0.599\pm 0.149}$& $0.602\pm 0.134$ &$ 0.675\pm 0.059$ & \underline{$0.604\pm 0.145$} \\
                    & 15  &  $0.676\pm 0.071$ & $0.700\pm 0.067$ & $0.671\pm 0.060$ & $0.676\pm 0.085$ & \underline{$0.667\pm 0.065$} & $0.700\pm 0.041$ & $\bf {0.621\pm0.155} $\\
                    & 20  & $0.727\pm 0.039$ & $0.749\pm 0.049$ & $\bf 0.695\pm 0.066$ & $0.747\pm 0.057$ & $0.713\pm 0.045$ & $0.727\pm 0.039$ & \underline{$0.709\pm 0.113$} \\
 \midrule
 \multirow{6}{*}{Computer} 
    &0 & $0.397\pm 0.016$ & $0.397\pm 0.015$ & $0.411\pm 0.029$  & \underline{$0.384\pm 0.018$} & $\bf {0.378\pm 0.014}$ & $0.434\pm0.021 $&  $0.395\pm 0.018 $\\
    &1  & $0.412\pm 0.020$ & $0.413\pm 0.018$ & $0.433\pm 0.028 $& \underline{$0.399\pm 0.022$} & $\bf {0.395\pm 0.023}$ & $0.448\pm 0.025$&   $0.406\pm 0.020$ \\
    &2  & $0.442\pm 0.026$ & $0.447\pm 0.030$ & $0.478\pm 0.032$ & $\bf{ 0.431\pm 0.024}$ & \underline{$0.432\pm0.029$} & $0.478\pm0.027$&  $0.437\pm 0.020$ \\
    &3  &$ 0.492\pm 0.041$ & $0.500\pm 0.039$ & $0.515\pm 0.039$ & $0.482\pm 0.037$ & \underline{$0.481\pm 0.036$} & $0.511\pm 0.026$ &$ \bf {0.460\pm 0.032}$\\
    &4  & $0.541\pm0.054$ & $0.566\pm 0.060$ &$ 0.535\pm 0.041$ & $0.541\pm 0.047$ & \underline{$0.521\pm0.050$} & $0.550\pm 0.051$ & $\bf {0.491\pm 0.048}$\\
    &5  & $0.565\pm 0.043$ & $0.608\pm 0.055$ & $0.547\pm 0.023$ & $0.581\pm 0.053 $& \underline{$0.541\pm 0.034$} &$ 0.553\pm 0.040$& $\bf {0.524\pm 0.035}$ \\

\midrule
\multirow{6}{*}{Cal.housing} 
    &0  & $0.502\pm0.012$ & $0.507\pm 0.011$ & $0.502\pm 0.019$& \underline{$0.491\pm 0.010$} & $\bf {0.488\pm 0.010}$ & $0.505\pm 0.013$ & $0.491\pm 0.016$\\
    &1  & $0.519\pm 0.025 $& $0.526\pm 0.026$ & $0.517\pm 0.026$ &\underline{$0.512\pm 0.028$} & $\bf{0.506\pm 0.024}$ & $0.523\pm 0.021$ & $0.513\pm 0.030$\\
    &2 & $0.562\pm 0.032$ & $0.603\pm 0.060$ & $0.553\pm 0.025$ & $0.581\pm 0.063$ & \underline{$0.547\pm 0.030$} & $0.597\pm 0.049$& $\bf{0.539\pm0.028}$ \\
    &3  & $0.613\pm 0.043$ & $0.669\pm 0.065$ & \underline{$0.573\pm 0.029$} & $0.648\pm0.074$ & $0.578\pm 0.031$ & $0.650\pm 0.059$& $\bf {0.552\pm 0.049}$\\
    &4  & $0.687\pm 0.058$ & $0.747\pm 0.052$ & \underline{$0.598\pm 0.036$} & $0.746\pm 0.060$ & $0.615\pm 0.042$ & $0.660\pm 0.047$ & $\bf{ 0.576\pm 0.049}$\\
    &5  & $0.729\pm 0.052$ & $0.790\pm 0.022$ & \underline{$0.619\pm 0.040$} & $0.792\pm 0.021$ & $0.638\pm 0.032 $& $0.655\pm 0.057$ & $\bf {0.605\pm 0.053}$\\
\bottomrule
\end{tabular}}
\label{outlier}
\end{table*}

\subsubsection{Results Analysis}
First, our method exhibits optimal or sub-optimal performance on most datasets with or without outliers, especially on Contact-Lenses and Machine datasets
, as shown in Table \ref{labelnoise} and Table \ref{outlier}.
Moreover, as shown in Fig. \ref{CDF}, the average rank of our method is first among all methods. We also observe that as the number of outliers increases, the average ranking of our method shows an upward trend. This demonstrates that our method outperforms the 7 baselines and is more robust against the influence of outliers.
Secondly, as the number of outliers increases, the generalization performance of the model shows a downward trend due to the presence of outliers in the training set.
However, our method exhibits slower degradation in generalization performance as the number of outliers increases, indicating its resilience to outlier influence.
Lastly, the  performance of PINSVOR is unsatisfactory because it primarily addresses noise near the classification boundary while disregarding the impact of outliers.

\begin{figure*}[htbp]
  \centering
  \subfigure[0\% outliers]{
    \includegraphics[width=0.3\textwidth]{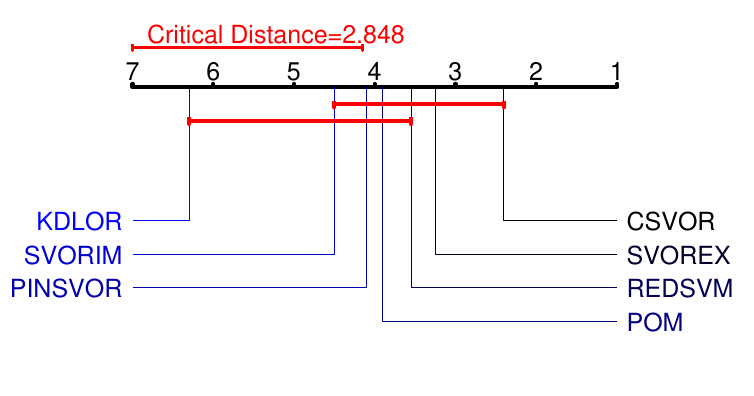}
 }
  \subfigure[1\% outliers]{
    \includegraphics[width=0.3\textwidth]{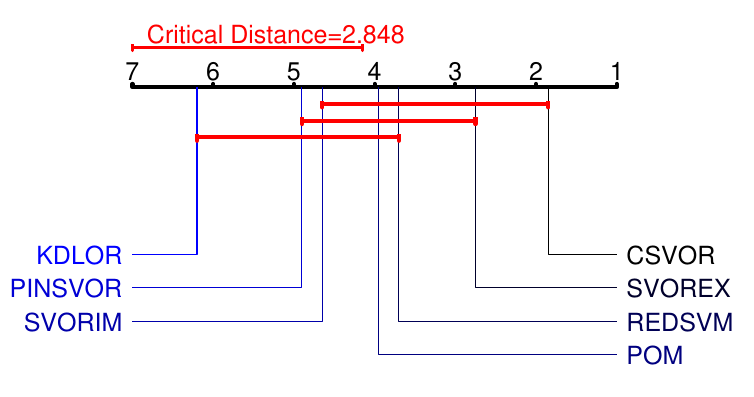}
  }
  \subfigure[5\% outliers]{
    \includegraphics[width=0.3\textwidth]{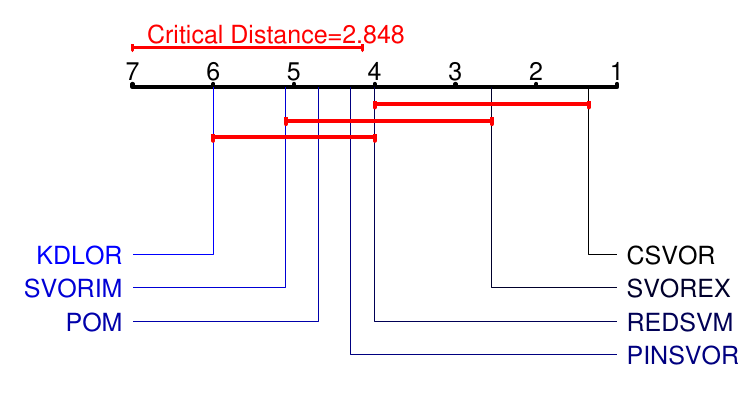}
  }

  \subfigure[10\% outliers ]{
    \includegraphics[width=0.3\textwidth]{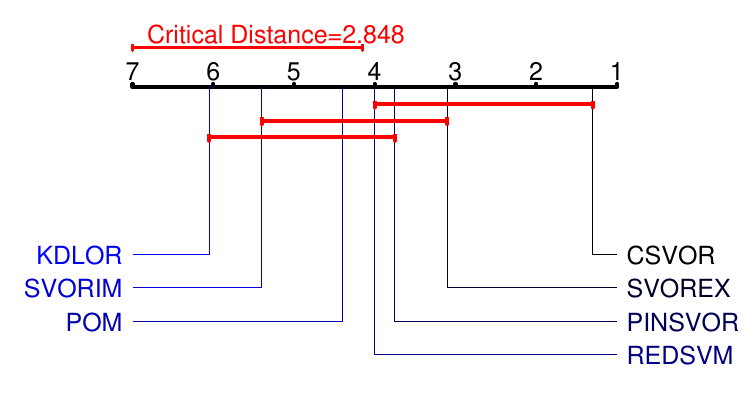}
  }
  \subfigure[15\% outliers ]{
    \includegraphics[width=0.3\textwidth]{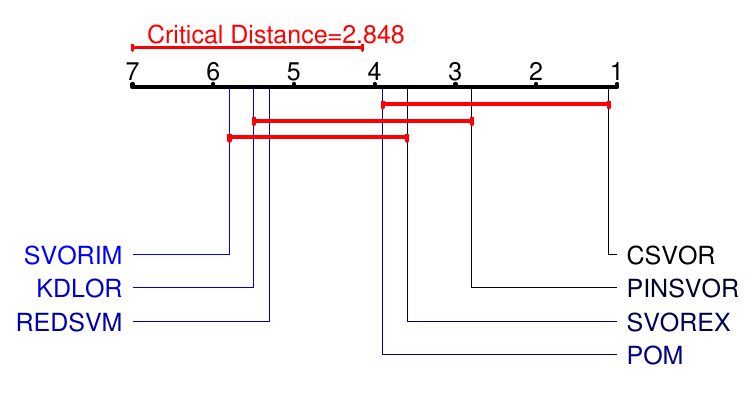}
  }
   \subfigure[20\% outliers ]{
    \includegraphics[width=0.3\textwidth]{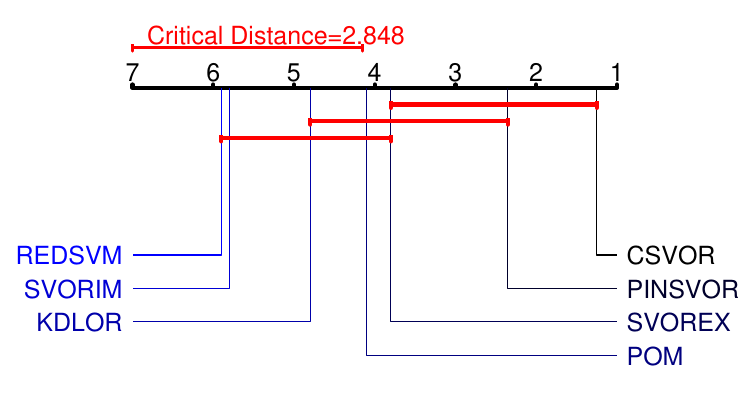}
  }
  \caption{The CD diagrams of the Friedman test with Nemenyi post-hoc test at the 5\% significance level. }
  \label{CDF}
\end{figure*}

\begin{figure*}[htbp]
  \centering

  \subfigure[Contact-Lenses]{
    \includegraphics[width=0.3\textwidth]{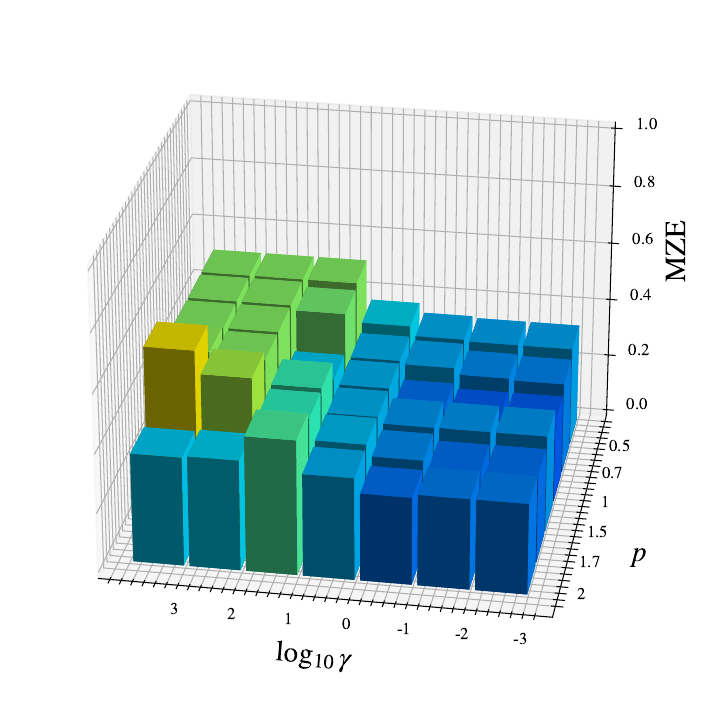}
 }
  \subfigure[Squash-Unstored]{
    \includegraphics[width=0.3\textwidth]{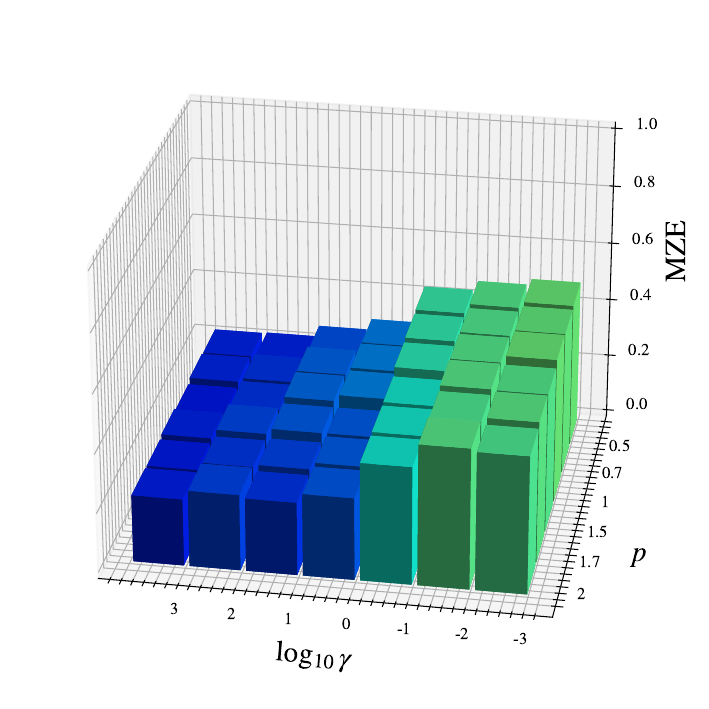}
  }
  \subfigure[Tae]{
    \includegraphics[width=0.3\textwidth]{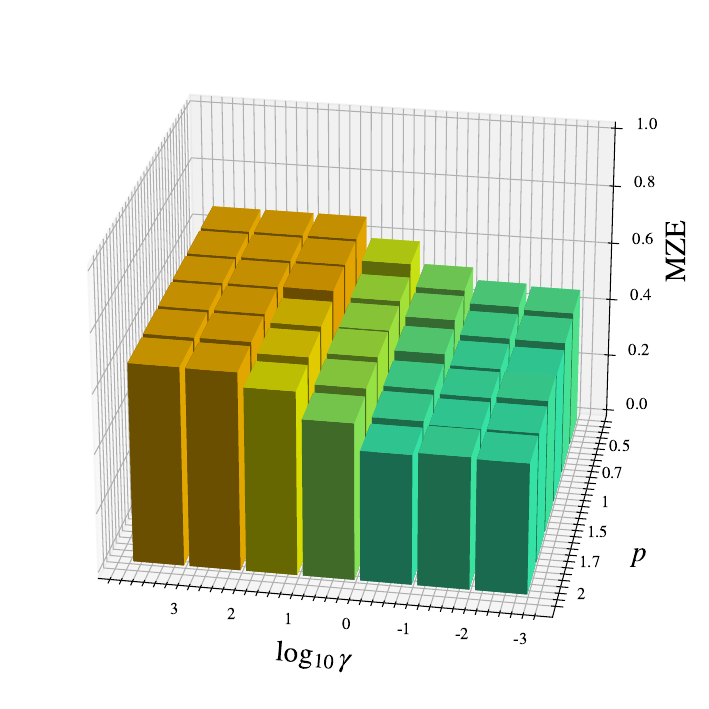}
  }

  \subfigure[Newthyroid]{
    \includegraphics[width=0.3\textwidth]{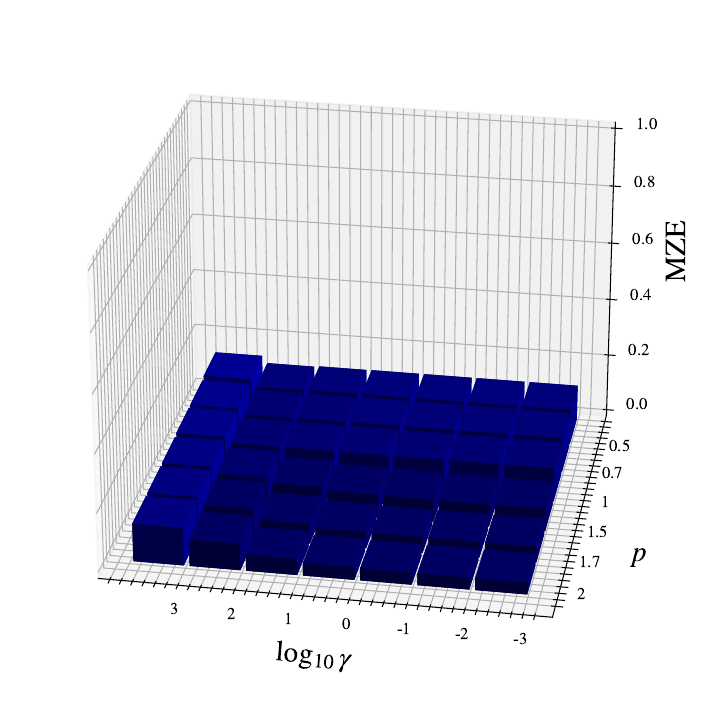}
 }
  \subfigure[Machine]{
    \includegraphics[width=0.3\textwidth]{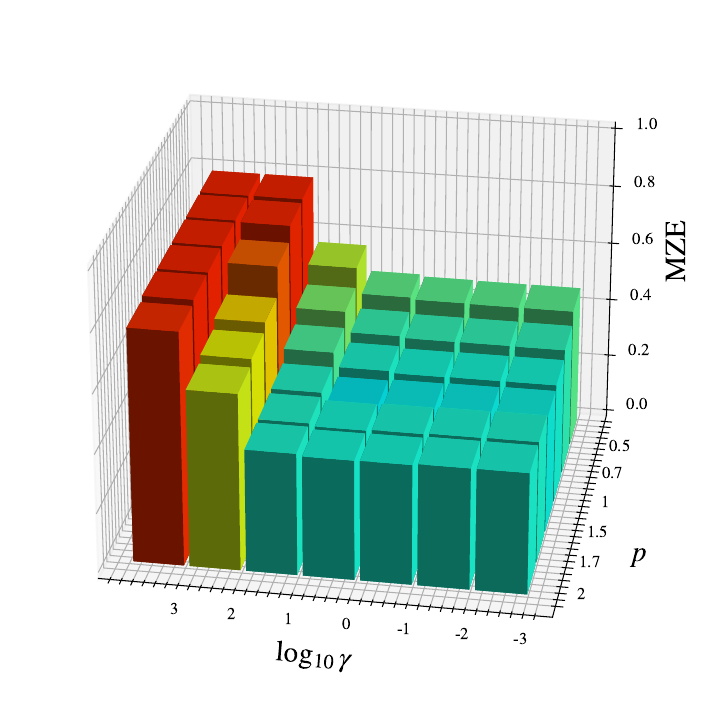}
  }
  \subfigure[Housing]{
    \includegraphics[width=0.3\textwidth]{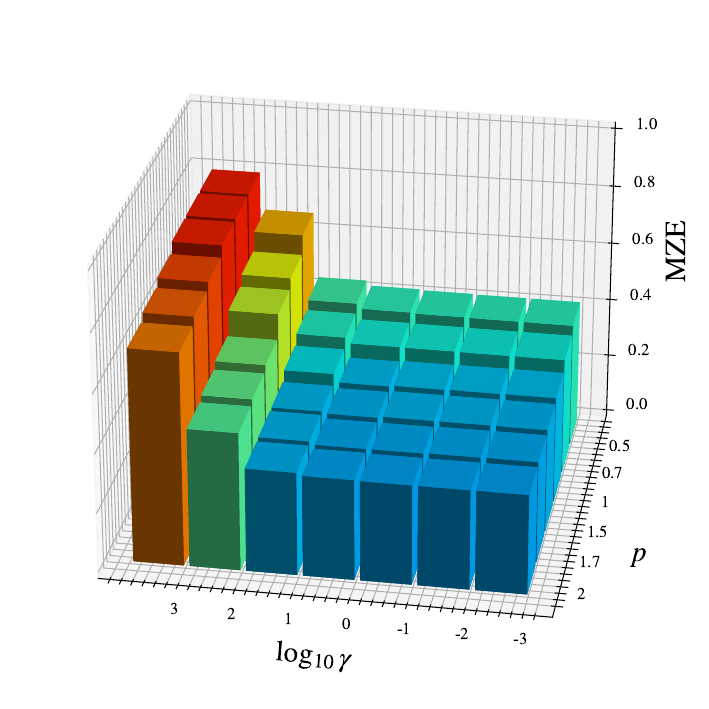}
  }

  \subfigure[Stock]{
    \includegraphics[width=0.3\textwidth]{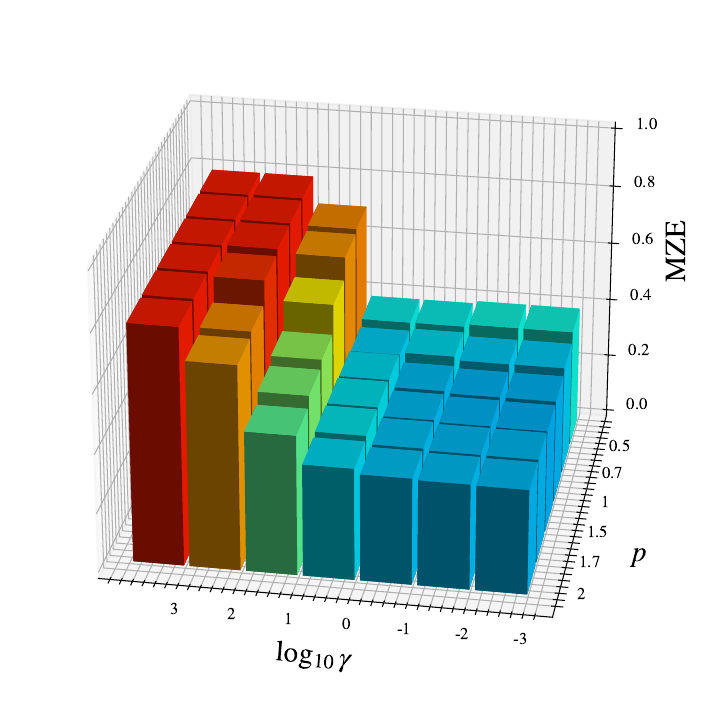}
  }
  \subfigure[Bank]{
    \includegraphics[width=0.3\textwidth]{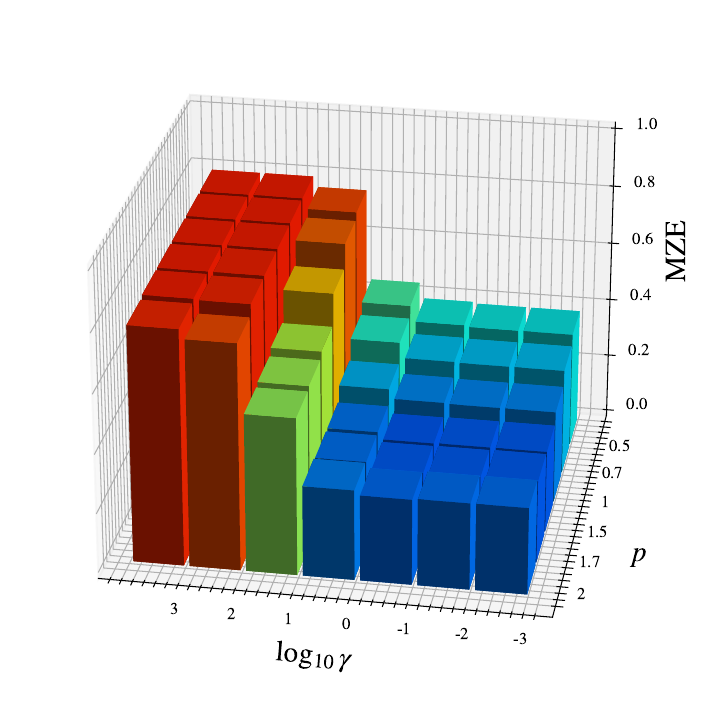}
  }
   \subfigure[Cal.housing ]{
    \includegraphics[width=0.3\textwidth]{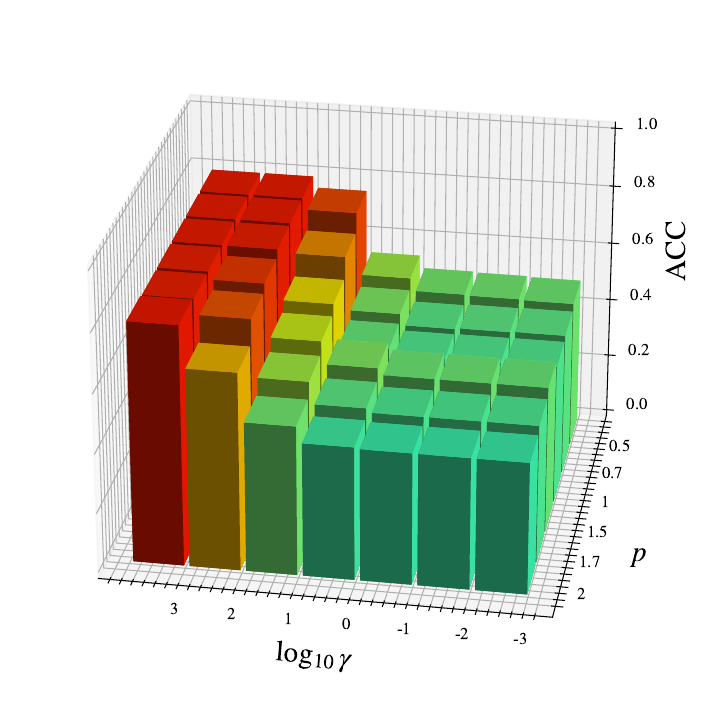}
  }
  \caption{MZE with respect to parameters $p$ and $\gamma$ on nine datasets.}
  \label{paramF}
\end{figure*}

\begin{figure}[htbp]
  \centering
  \subfigure[Contact-Lenses]{
    \includegraphics[width=0.3\textwidth]{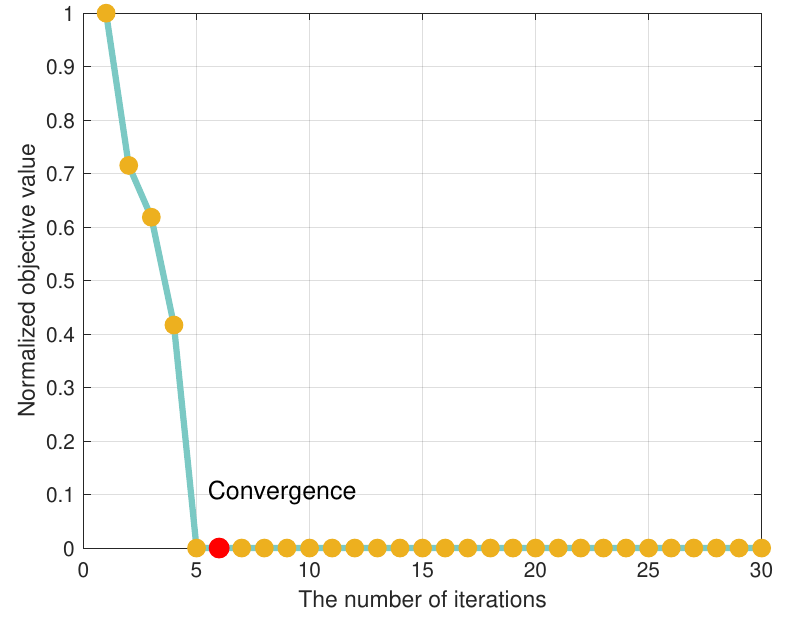}
  }
  \subfigure[Tae]{
    \includegraphics[width=0.3\textwidth]{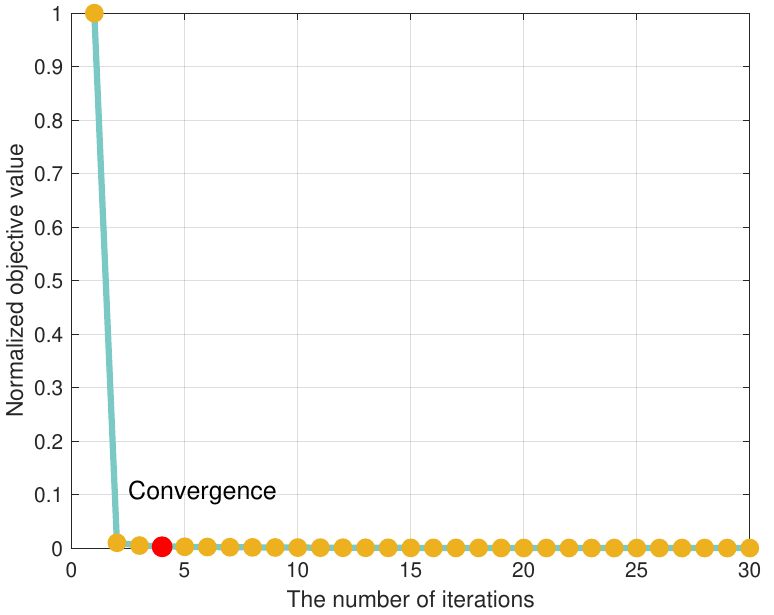}
  }
  \subfigure[Newthyroid]{
    \includegraphics[width=0.3\textwidth]{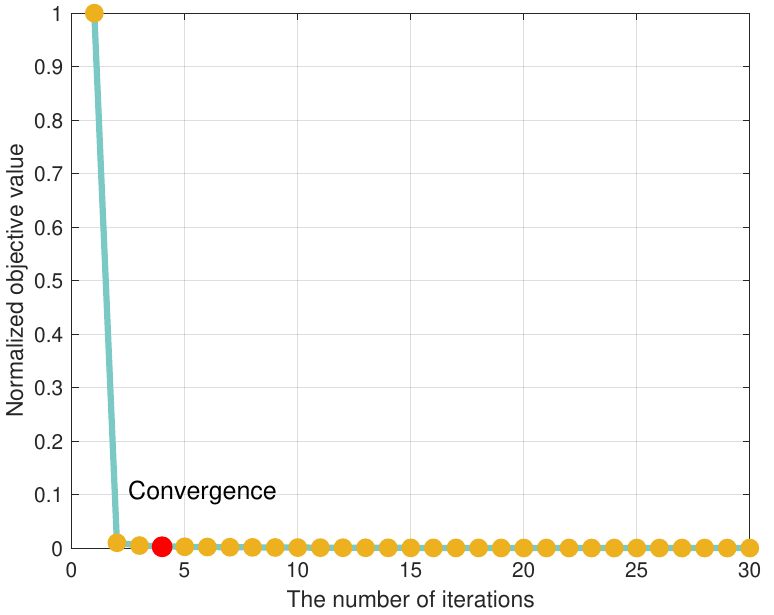}
  }

  \subfigure[Machine]{
    \includegraphics[width=0.3\textwidth]{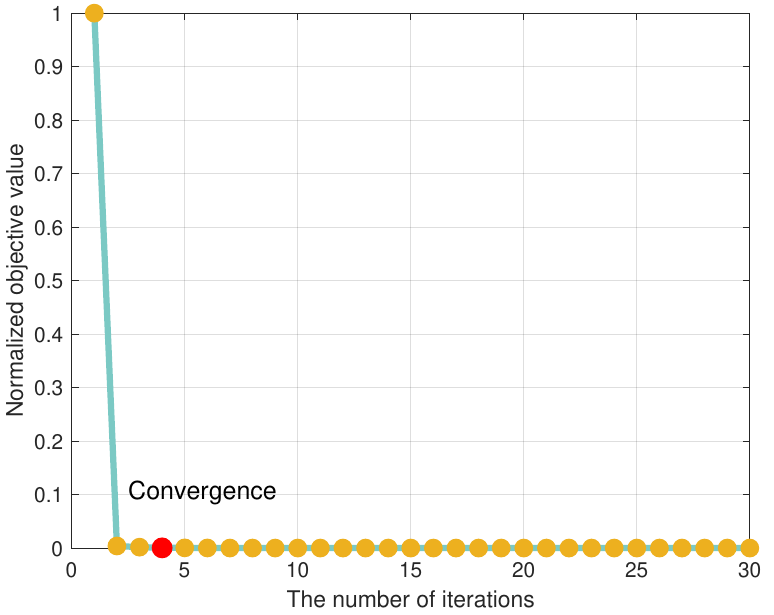}
  }
  \subfigure[Housing]{
    \includegraphics[width=0.3\textwidth]{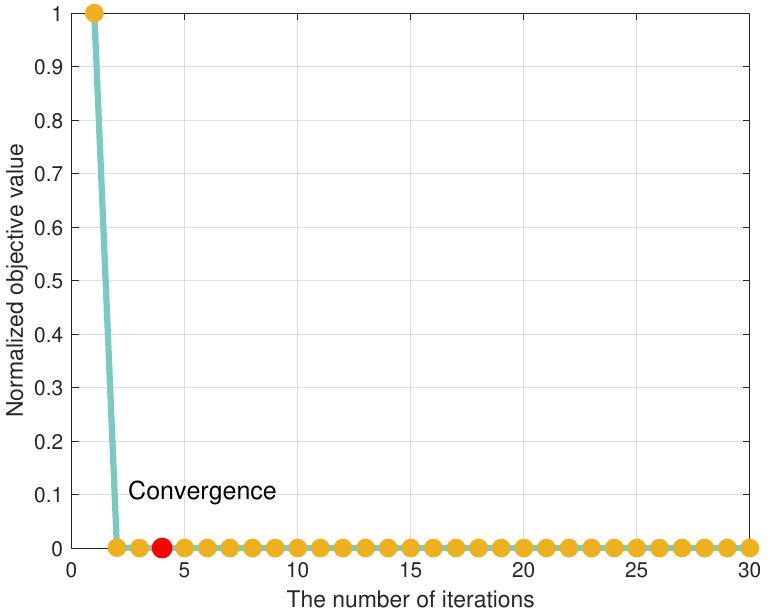}
  }
  \subfigure[Stock]{
    \includegraphics[width=0.3\textwidth]{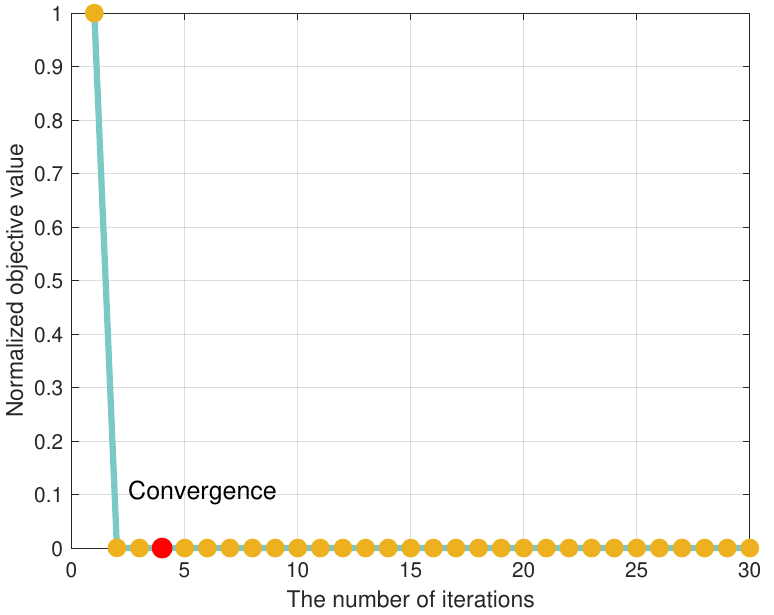}
  }

  \subfigure[Bank]{
    \includegraphics[width=0.3\textwidth]{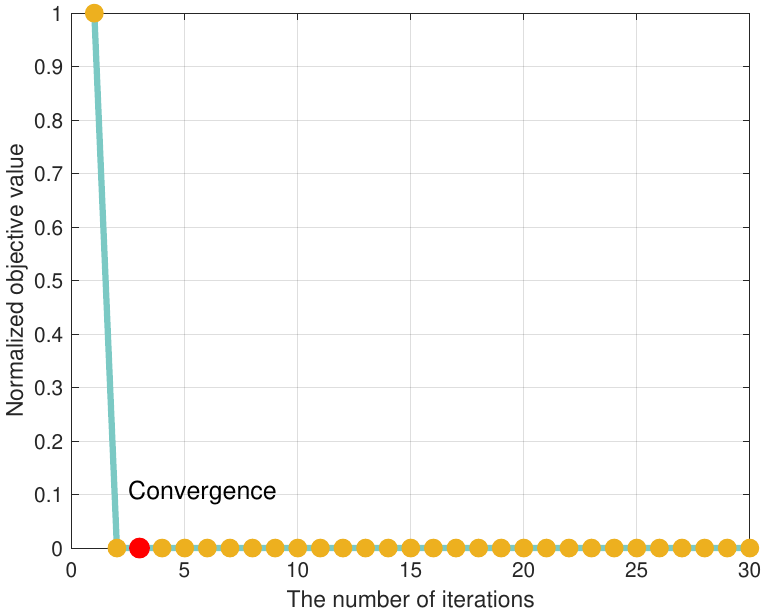}
  }
  \subfigure[Computer]{
    \includegraphics[width=0.3\textwidth]{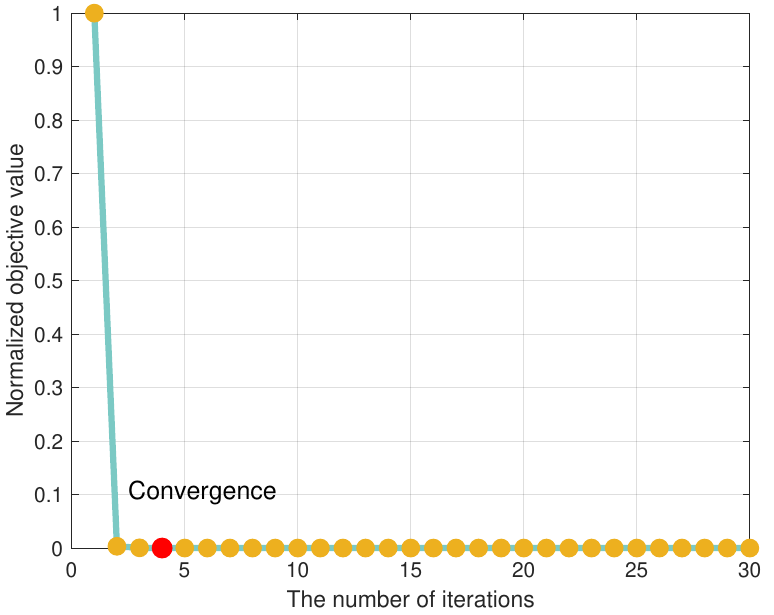}
  }
  \subfigure[Cal.housing]{
    \includegraphics[width=0.3\textwidth]{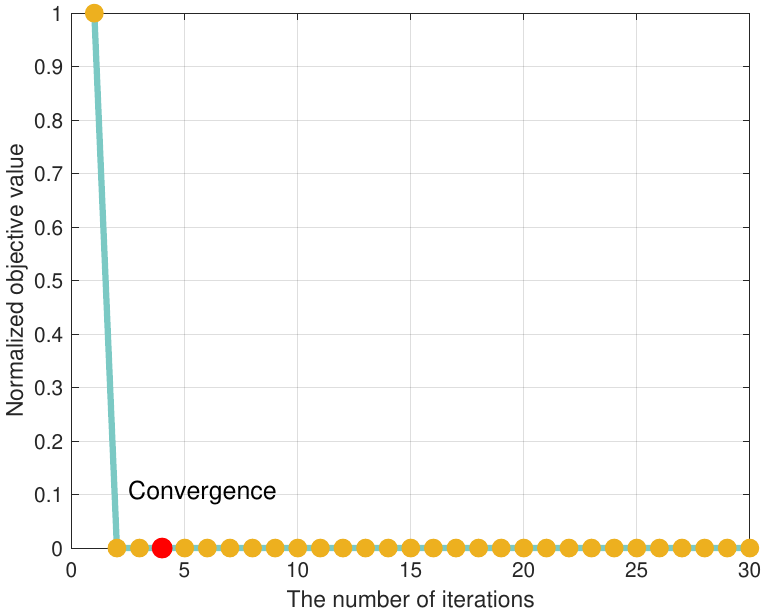}
  }
  \caption{Convergence curves of CSVOR on nine real-world
datasets.}
  \label{convergeF}
\end{figure}

\subsection{Visual experiment}

\begin{figure}
    \centering
    \includegraphics[scale=0.5]{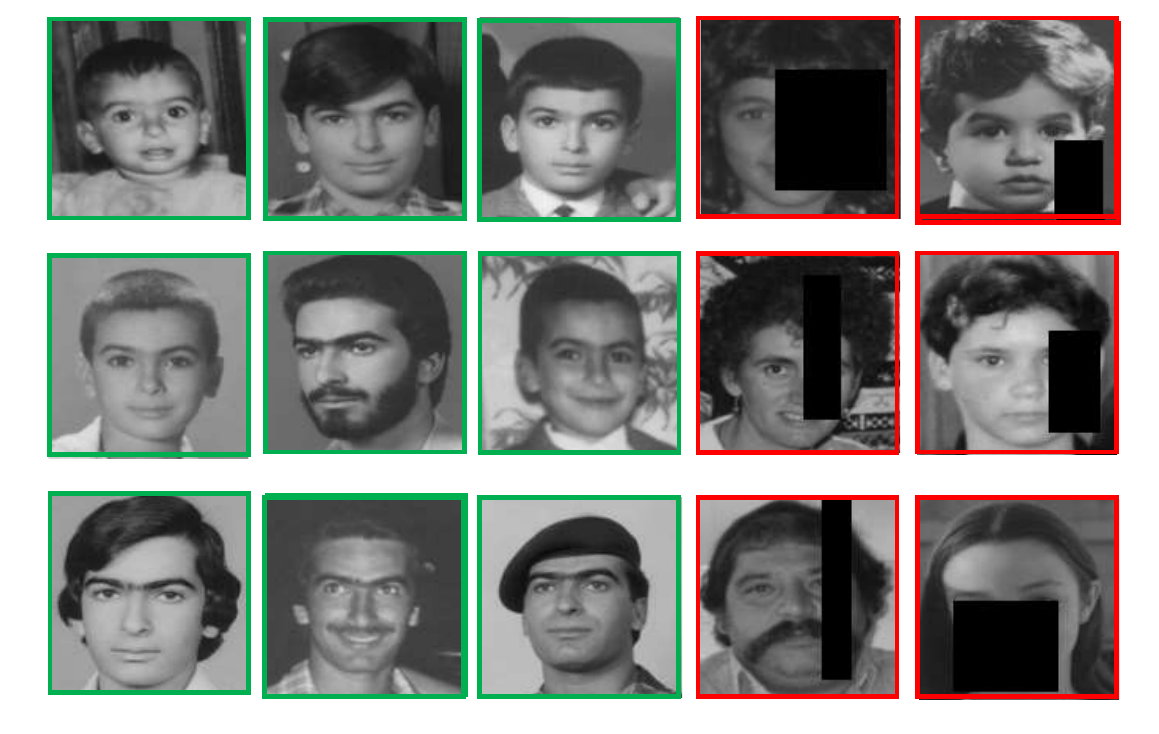}
    \caption{Visualized experimental results on the FG-NET dataset, the first three columns are normal samples, and the last two columns are abnormal samples with black occlusion blocks added.}
    \label{V1}
\end{figure}

In Section 5.2, we provide an intuitive explanation of why CSVOR is robust to outliers. 
CSVOR utilizes a weight matrix $D$ to detect and  remove outliers implicitly. This matrix marks detected outliers as 0, effectively removing them from consideration during training process.
The ability of the weight matrix $D$ to identify anomalies directly influences the performance of CSVOR. To illustrate this capability,
we conduct visual experiments on the FG-NET dataset, 
widely used for age estimation. 
This dataset comprises 1002 images of 82 individuals, 
ranging in age from 0 to 69 years old \cite{lanitis2002toward}. For the experiment, we randomly select 10\% of the data and introduce black occlusion blocks of varying sizes and positions. 
Fig. \ref{V1} presents the visualization results, where red boxes indicate abnormal points detected by the model, and green boxes represent normal points. 
The model successfully identifies images with black occlusion patches as outliers, supporting our previous findings and highlighting the effectiveness of the weight matrix $D$ in anomaly detection.
\subsection{Parameters Sensitivity}
In this subsection, we study the sensitivity of the parameters $p$ and $\gamma$ in CSVOR. Fig. \ref{paramF} shows the Mean Zero-One Error (MZE) on the test sets corresponding to different parameter combinations on nine datasets. (Due to space limitations, we only show the results on nine datasets.) We observe that different parameter combinations have different effects on the performance of the model, and the optimal parameter combinations corresponding to different datasets are different. Therefore, we recommend using five fold cross-validation to select appropriate parameters.

\subsection{Convergence Analysis}
To address the optimization problem involved in CSVOR, we propose an efficient optimization algorithm based on the Re-Weighted algorithm. The corresponding convergence analysis is given in Section 5.1.
Fig. \ref{convergeF} displays the convergence curve of our algorithm on nine datasets. It is evident that CSVOR converges quickly, typically within 10 iterations, across all datasets.
The rapid convergence of CSVOR, as illustrated in the convergence curves, highlights the efficiency and effectiveness of our optimization algorithm.

\section{CONCLUSION}
In this paper, we introduce CSVOR, a robust ordinal regression model that utilizes a novel  capped $\ell_p$-norm loss function designed to handle both light and heavy outliers. We provide an intuitive explanation for the robustness of CSVOR to outliers
CSVOR uses a  weight matrix $D$ to identify and eliminate the outliers implicitly during training process.
To solve the optimization problem associated with CSVOR, we propose an effective Re-Weighted optimization algorithm, supported by theoretical results demonstrating its convergence. Extensive experiments validate the effectiveness of CSVOR, showcasing its robust performance across various datasets.

Further research is needed to enhance our proposed model. As our current model is linear, its ability to fit complex functions is limited. Therefore, our future research will concentrate on developing nonlinear versions of our model. Additionally,  any unbounded loss may be affected by outliers.
Although our focus in this paper is on the ordinal hinge loss, the concepts presented herein can be extended to other ordinal regression losses. We plan to explore the application of capped $\ell_p$-norm to other ordinal losses, such as ordinal exponential loss and ordinal logistic loss, in future research.




 \appendix
 \bibliographystyle{elsarticle-num} 
\biboptions{sort&compress}
  \bibliography{elsarticletemplatenum.bib}

\begin{thebibliography}{10}
\expandafter\ifx\csname url\endcsname\relax
  \def\url#1{\texttt{#1}}\fi
\expandafter\ifx\csname urlprefix\endcsname\relax\def\urlprefix{URL }\fi
\expandafter\ifx\csname href\endcsname\relax
  \def\href#1#2{#2} \def\path#1{#1}\fi

\bibitem{gutierrez2015ordinal}
P.~A. Guti{\'e}rrez, M.~Perez-Ortiz, J.~Sanchez-Monedero, F.~Fernandez-Navarro, C.~Hervas-Martinez, Ordinal regression methods: survey and experimental study, IEEE Transactions on Knowledge and Data Engineering 28~(1) (2015) 127--146.

\bibitem{cruz2014metrics}
M.~Cruz-Ram{\'\i}rez, C.~Herv{\'a}s-Mart{\'\i}nez, J.~S{\'a}nchez-Monedero, P.~A. Guti{\'e}rrez, Metrics to guide a multi-objective evolutionary algorithm for ordinal classification, Neurocomputing 135 (2014) 21--31.

\bibitem{pfannschmidt2020feature}
L.~Pfannschmidt, J.~Jakob, F.~Hinder, M.~Biehl, P.~Tino, B.~Hammer, Feature relevance determination for ordinal regression in the context of feature redundancies and privileged information, Neurocomputing 416 (2020) 266--279.

\bibitem{diaz2019soft}
R.~Diaz, A.~Marathe, Soft labels for ordinal regression, in: in Proc. IEEE/CVF Conference on Computer Vision and Pattern Recognition, 2019, pp. 4738--4747.

\bibitem{yan2014cost}
H.~Yan, Cost-sensitive ordinal regression for fully automatic facial beauty assessment, Neurocomputing 129 (2014) 334--342.

\bibitem{vargas2020cumulative}
V.~M. Vargas, P.~A. Guti{\'e}rrez, C.~Hervas-Martinez, Cumulative link models for deep ordinal classification, Neurocomputing 401 (2020) 48--58.

\bibitem{niu2016ordinal}
Z.~Niu, M.~Zhou, L.~Wang, X.~Gao, G.~Hua, Ordinal regression with multiple output cnn for age estimation, in: in Proc. IEEE Conference on Computer Vision and Pattern Recognition, 2016, pp. 4920--4928.

\bibitem{5995437}
K.-Y. Chang, C.-S. Chen, Y.-P. Hung, Ordinal hyperplanes ranker with cost sensitivities for age estimation, in: in Proc. IEEE/CVF Conference on Computer Vision and Pattern Recognition, 2011, pp. 585--592.

\bibitem{tian2014comparative}
Q.~Tian, S.~Chen, X.~Tan, Comparative study among three strategies of incorporating spatial structures to ordinal image regression, Neurocomputing 136 (2014) 152--161.

\bibitem{li2022ordinalclip}
W.~Li, X.~Huang, Z.~Zhu, Y.~Tang, X.~Li, J.~Zhou, J.~Lu, Ordinalclip: Learning rank prompts for language-guided ordinal regression, in: in Proc.Advances in Neural Information Processing Systems, 2022, pp. 35313--35325.

\bibitem{qian2022contrastive}
T.~Qian, F.~Li, M.~Zhang, G.~Jin, P.~Fan, W.~Dai, Contrastive learning from label distribution: A case study on text classification, Neurocomputing 507 (2022) 208--220.

\bibitem{burkner2019ordinal}
P.-C. Bürkner, M.~Vuorre, Ordinal regression models in psychology: A tutorial, Advances in Methods and Practices in Psychological Science 2~(1) (2019) 77--101.

\bibitem{tang2023disease}
W.~Tang, Z.~Yang, Y.~Song, Disease-grading networks with ordinal regularization for medical imaging, Neurocomputing 545 (2023) 126245.

\bibitem{zhu2021convolutional}
H.~Zhu, H.~Shan, Y.~Zhang, L.~Che, X.~Xu, J.~Zhang, J.~Shi, F.-Y. Wang, Convolutional ordinal regression forest for image ordinal estimation, IEEE Transactions on Neural Networks and Learning Systems 33~(8) (2021) 4084--4095.

\bibitem{shin2022moving}
N.-H. Shin, S.-H. Lee, C.-S. Kim, Moving window regression: A novel approach to ordinal regression, in: in Proc. IEEE/CVF Conference on Computer Vision and Pattern Recognition, 2022, pp. 18760--18769.

\bibitem{pratiwi2019implementing}
N.~Pratiwi, et~al., Implementing ordinal regression model for analyzing happiness level in indonesia, in: Journal of Physics: Conference Series, Vol. 1320, IOP Publishing, 2019, pp. 012--015.

\bibitem{sun2023robust}
J.~Sun, Z.~Gong, D.~Zhang, Y.~Xu, G.~Wei, A robust ordinal regression feedback consensus model with dynamic trust propagation in social network group decision-making, Information Fusion 100 (2023) 101952.

\bibitem{chu2005new}
W.~Chu, S.~S. Keerthi, New approaches to support vector ordinal regression, in: in Proc. International Conference on Machine Learning, 2005, pp. 145--152.

\bibitem{zhong2023ordinal}
G.~Zhong, Y.~Xiao, B.~Liu, L.~Zhao, X.~Kong, Ordinal regression with pinball loss, IEEE Transactions on Neural Networks and Learning Systems (2023) 1--15.

\bibitem{wang2023enhanced}
J.~Wang, X.~Zhang, F.~Nie, X.~Li, Enhanced robust fuzzy k-means clustering joint $\ell_0$-norm constraint, Neurocomputing 561 (2023) 126842.

\bibitem{zhang2020towards}
X.-Y. Zhang, C.-L. Liu, C.~Y. Suen, Towards robust pattern recognition: A review, Proceedings of the IEEE 108~(6) (2020) 894--922.

\bibitem{nie2017multiclass}
F.~Nie, X.~Wang, H.~Huang, Multiclass capped $\ell_p$-norm svm for robust classifications, in: in Proc. AAAI Conference on Artificial Intelligence, 2017, pp. 2415--2421.

\bibitem{wang2022capped}
Z.~Wang, H.~Hu, R.~Wang, Q.~Zhang, F.~Nie, X.~Li, Capped $\ell_p$-norm linear discriminant analysis for robust projections learning, Neurocomputing 511 (2022) 399--409.

\bibitem{frank2001simple}
E.~Frank, M.~Hall, A simple approach to ordinal classification, in: in Proc. European Conference on Machine Learning, Springer, 2001, pp. 145--156.

\bibitem{waegeman2009ensemble}
W.~Waegeman, L.~Boullart, et~al., An ensemble of weighted support vector machines for ordinal regression, International Journal of Computer Systems Science and Engineering 3~(1) (2009) 47--51.

\bibitem{perez2013projection}
M.~Perez-Orti{\'z}, P.~A. Guti{\'e}rrez, C.~Herv{\'a}s-Mart{\'\i}nez, Projection-based ensemble learning for ordinal regression, IEEE transactions on Cybernetics 44~(5) (2013) 681--694.

\bibitem{jiang2021non}
H.~Jiang, Z.~Yang, Z.~Li, Non-parallel hyperplanes ordinal regression machine, Knowledge-Based Systems 216 (2021) 106593.

\bibitem{gouvert2020ordinal}
O.~Gouvert, T.~Oberlin, C.~F{\'e}votte, Ordinal non-negative matrix factorization for recommendation, in: International Conference on Machine Learning, PMLR, 2020, pp. 3680--3689.

\bibitem{shi2023deep}
X.~Shi, W.~Cao, S.~Raschka, Deep neural networks for rank-consistent ordinal regression based on conditional probabilities, Pattern Analysis and Applications 26~(3) (2023) 941--955.

\bibitem{liu2020unimodal}
X.~Liu, F.~Fan, L.~Kong, Z.~Diao, W.~Xie, J.~Lu, J.~You, Unimodal regularized neuron stick-breaking for ordinal classification, Neurocomputing 388 (2020) 34--44.

\bibitem{li2022unimodal}
Q.~Li, J.~Wang, Z.~Yao, Y.~Li, P.~Yang, J.~Yan, C.~Wang, S.~Pu, Unimodal-concentrated loss: Fully adaptive label distribution learning for ordinal regression, in: in Proc. IEEE/CVF Conference on Computer Vision and Pattern Recognition, 2022, pp. 20513--20522.

\bibitem{shashua2002ranking}
A.~Shashua, A.~Levin, Ranking with large margin principle: Two approaches, Advances in neural information processing systems 15 (2002).

\bibitem{chu2007support}
W.~Chu, S.~S. Keerthi, Support vector ordinal regression, Neural Computation 19~(3) (2007) 792--815.

\bibitem{lin2012reduction}
H.-T. Lin, L.~Li, Reduction from cost-sensitive ordinal ranking to weighted binary classification, Neural Computation 24~(5) (2012) 1329--1367.

\bibitem{xiang2012discriminative}
S.~Xiang, F.~Nie, G.~Meng, C.~Pan, C.~Zhang, Discriminative least squares regression for multiclass classification and feature selection, IEEE transactions on neural networks and learning systems 23~(11) (2012) 1738--1754.

\bibitem{mccullagh1980regression}
P.~McCullagh, Regression models for ordinal data, Journal of the Royal Statistical Society: Series B (Methodological) 42~(2) (1980) 109--127.

\bibitem{sun2009kernel}
B.-Y. Sun, J.~Li, D.~D. Wu, X.-M. Zhang, W.-B. Li, Kernel discriminant learning for ordinal regression, IEEE Transactions on Knowledge and Data Engineering 22~(6) (2009) 906--910.

\bibitem{sanchez2019orca}
J.~S{\'a}nchez-Monedero, P.~A. Guti{\'e}rrez, M.~P{\'e}rez-Ortiz, Orca: A matlab/octave toolbox for ordinal regression, Journal of Machine Learning Research 20~(125) (2019) 1--5.

\bibitem{liu2022multi}
D.~Liu, J.~Zhao, J.~Wu, G.~Yang, F.~Lv, Multi-category classification with label noise by robust binary loss, Neurocomputing 482 (2022) 14--26.

\bibitem{frenay2013classification}
B.~Frenay, M.~Verleysen, Classification in the presence of label noise: a survey, IEEE transactions on Neural Networks and Learning Systems 25~(5) (2013) 845--869.

\bibitem{demvsar2006statistical}
J.~Dem{\v{s}}ar, Statistical comparisons of classifiers over multiple data sets, The Journal of Machine learning research 7 (2006) 1--30.

\bibitem{xu2017robust}
G.~Xu, Z.~Cao, B.-G. Hu, J.~C. Principe, Robust support vector machines based on the rescaled hinge loss function, Pattern Recognition 63 (2017) 139--148.

\bibitem{lanitis2002toward}
A.~Lanitis, C.~J. Taylor, T.~F. Cootes, Toward automatic simulation of aging effects on face images, IEEE Transactions on pattern Analysis and machine Intelligence 24~(4) (2002) 442--455.

\end{thebibliography}


\end{document}